\documentclass[11pt]{article}
\usepackage{fullpage}
\usepackage{amsmath,amsthm,amssymb}
\usepackage{algorithm}
\usepackage{algorithmic}
\usepackage{subfigure}
\usepackage{graphicx}
\usepackage{pbox}
\usepackage{multirow}
\usepackage[square,sort,comma,numbers]{natbib}
\usepackage{footnote}

\usepackage{color}

\newcommand{\vct}{\boldsymbol }

\newcommand{\nml}{\mathcal{N}}

\newcommand{\argmin}{\mathrm{argmin}}
\newcommand{\argmax}{\mathrm{argmax}}

\newcommand{\rank}{\mathrm{rank}}

\newcommand{\diag}{\mat{diag}}

\newcommand{\mat}[1]{\mathbf{#1}}

\newcommand{\norm}[1]{\left\|#1\right\|}

\newtheorem{thm}{Theorem}[section]
\newtheorem{lem}{Lemma}[section]
\newtheorem{cor}{Corollary}[section]

\newtheorem{defn}{Definition}[section]

\newtheorem{conj}{Conjecture}[section]
\newtheorem{rem}{Remark}[section]

\title{An Improved Gap-Dependency Analysis of the Noisy Power Method}
\date{}

\author{
Maria Florina Balcan, Simon S. Du, Yining Wang, Adams Wei Yu\\
Machine Learning Department\\
Carnegie Mellon University\\
\texttt{\{ninamf, ssdu, yiningwa, weiyu\}@cs.cmu.edu} \\
}

\begin{document}
\maketitle

\begin{abstract}
	We consider the \emph{noisy power method} algorithm, which has wide applications in machine learning and statistics,
	especially those related to principal component analysis (PCA) under resource (communication, memory or privacy) constraints.
	Existing analysis of the noisy power method \citep{noisy-power-method,li2015rivalry} shows an unsatisfactory dependency over the 
	``consecutive" spectral gap $(\sigma_k-\sigma_{k+1})$ of an input data matrix, which could be very small and hence limits the algorithm's applicability.
	In this paper, we present a new analysis of the noisy power method that achieves improved gap dependency for both sample complexity and noise tolerance bounds.
	More specifically, we improve the dependency over $(\sigma_k-\sigma_{k+1})$ to dependency over $(\sigma_k-\sigma_{q+1})$, 
	where $q$ is an intermediate algorithm parameter and could be much larger than the target rank $k$.
	Our proofs are built upon a novel characterization of proximity between two subspaces that differ from canonical angle characterizations
	analyzed in previous works \citep{noisy-power-method,li2015rivalry}.
	Finally, we apply our improved bounds to distributed private PCA and memory-efficient streaming PCA and obtain bounds
	that are superior to existing results in the literature.
\end{abstract}


\section{Introduction}

Principal Component Analysis (PCA) is a fundamental problem in statistics and machine learning.
The objective of PCA is to find a small number of orthogonal directions in the $d$-dimensional Euclidean space $\mathbb R^d$
that have the highest variance of a given sample set.
Mathematically speaking, given a $d\times d$ positive semi-definite matrix $\mat A$ of interest ($\mat A$ is usually the sample covariance matrix
$\mat A=\frac{1}{n}\sum_{i=1}^n{\vct z_i\vct z_i^\top}$ for $n$ data points $\vct z_1,\cdots,\vct z_n$),
one wishes to find the top-$k$ eigen-space of $\mat A$, where $k$ is the number of principal directions of interest and is typically much smaller than the ambient dimension $d$.
A popular algorithm for computing PCA is the \emph{matrix power method}, which starts with a random $d\times p$ matrix ($p\geq k$) $\mat X_0$ with orthonormal columns
and iteratively performs the following computation for $\ell=1,\cdots,L$:
\begin{enumerate}
	\item \textbf{Subspace iteration}: $\mat Y_{\ell} = \mat A\mat X_{\ell-1}$.
	\item \textbf{QR factorization}: $\mat Y_{\ell} = \mat X_{\ell}\mat R_{\ell}$, where $\mat X_{\ell}\in\mathbb R^{d\times p}$ has orthonormal columns and $\mat R_{\ell}\in\mathbb R^{p\times p}$ 
	is an upper-triangular matrix.
\end{enumerate}
It is well-known that when the number of iterations $L$ is sufficiently large, the span of the output $\mat X_L$ can be arbitrarily close to $\mat U_k$, the top-$k$ eigen-space of $\mat A$;
that is, $\|(\mat I-\mat X_L\mat X_L^\top)\mat U_k\|_2\leq\epsilon$ for arbitrarily small $\epsilon > 0$.
One particular drawback of power method is that the rate of convergence depends on the \emph{consecutive} eigengap $(\sigma_k-\sigma_{k+1})$
when $p=k$ (i.e., $\mat X_\ell$ has exactly the same number of columns as the target rank $k$).
The consecutive eigengap could be very small for practical large-scale matrices.
As a remedy, practitioners generally set $p$ to be slightly larger than $k$ for faster convergence and numerical stability~\citep{musco2015stronger}.
\cite{mgu-subspace-iteration} formally justifies this process by proving that under mild conditions, the dependency on $(\sigma_k-\sigma_{k+1})$ could be improved
to the ``larger" spectral gap $(\sigma_k-\sigma_{q+1})$, for some $k\le q\le p$,
which may be significantly larger than the consecutive gap even if $q$ is at the same order of $k$.
\footnote{Sec.~\ref{sec:main} provides such an example matrix with power-law decaying spectrum.}
Despite the wide applicability and extensive analysis of the (exact) matrix power method, 
in practice it is sometimes desired to analyze a \emph{noisy} version of power method,
where each subspace iteration computation is corrupted with noise.
Such noise could come from resource constraints such as inherent machine precision or memory storage,
or artificially imposed constraints for additional objectives such as data privacy preservation.
In both cases, the noise model can be expressed as $\mat Y_\ell=\mat A\mat X_{\ell-1}+\mat G_\ell$,
where $\mat G_\ell$ is a $d\times p$ noise matrix for iteration $\ell$ that can be either stochastic or deterministic (adversarial).
Note that $\mat G_\ell$ could differ from iteration to iteration but
the QR factorization step $\mat Y_\ell=\mat X_\ell\mat R_\ell$ is still assumed to be exact.
The noisy power method has attracted increasing interest from both machine learning and theoretical computer science societies
due to its simplicity and broad applicability \citep{noisy-power-method,li2015rivalry,musco2015stronger,mitliagkas2013memory}.
In particular, \citep{noisy-power-method} establishes both convergence guarantees and error tolerance (i.e., the largest magnitude of the noise matrix $\mat G_\ell$ 
the algorithm allows to produce consistent estimates of $\mat U_k$)
of the noisy power method.
\citep{noisy-power-method} also applied their results to PCA with resource (privacy, memory) constraints
and obtained improved bounds over existing results.

\subsection{Our contributions}

\paragraph{Improved gap dependency analysis of the noisy power method}
Our main contribution is a new analysis of the noisy power method with improved gap dependency.
More specifically, 
we improve the prior gap dependency $(\sigma_k-\sigma_{k+1})$ to $(\sigma_k-\sigma_{q+1})$,
where $q$ is certain integer between the target rank $k$ and the number of columns used in subspace iteration $p$.
Our results partially solve a open question in \citep{noisy-power-method}, which conjectured that such improvement over gap dependency
should be possible if $p$ is larger than $k$.
To our knowledge, our bounds are the first to remove dependency over the consecutive spectral gap $(\sigma_k-\sigma_{k+1})$ for the noisy power method.

\paragraph{Gap-independent bounds}
As a by-product of our improved gap dependency analysis, we apply techniques in a recent paper \citep{musco2015stronger}
to obtain \emph{gap-independent} bounds for the approximation error $\|\mat A-\mat X_L\mat X_L^\top\mat A\|_2$.
This partially addresses another conjecture in \citep{noisy-power-method} regarding gap-independent approximation error bounds
with slightly worse bounds on magnitude of error matrices $\mat G_\ell$.

\paragraph{Applications}
The PCA problem has been previously considered under various resource constraints.
Two particularly important directions are private PCA \citep{private-pca-incoherent,analyze-gauss,exponential-pca,noisy-power-method},
where privacy of the data matrix being analyzed is formally preserved,
and distributed PCA \citep{distributed-pca,optimal-distributed-pca} where data matrices are stored separately on several machines
and communications among machines are constrained.
In this paper we propose a \emph{distributed private PCA} problem that unifies these two settings.
Our problem includes the entrywise private PCA setting in \citep{private-pca-incoherent,noisy-power-method} and distributed PCA setting in \citep{distributed-pca}
as special cases
and we demonstrate improved bounds over existing results for both problems.

We also apply our results to the memory-efficient streaming PCA problem considered in \citep{noisy-power-method,li2015rivalry,mitliagkas2013memory},
where data points arrive in streams and the algorithm is only allowed to use memory proportional to the size of the final output.
Built upon our new analysis of the noisy power method we improve state-of-the-art sample complexity bounds obtained in \citep{noisy-power-method}.

\paragraph{Proof techniques}
The noisy power method poses unique challenges for a improved gap dependency analysis.
Previous such analysis for the \emph{exact} power method in \citep{mgu-subspace-iteration,simon-spectral-gap}
considers a variant of the algorithm that only computes QR decomposition after the last subspace iteration.
Such strategy is no longer valid for noisy power method because without per-iteration QR decomposition,
the noise $\mat G_\ell$ will aggregate across iterations and eventually breaks the proximity between the final output $\mat X_L$
and the target top-$k$ eigen-space $\mat U_k$.
In the analysis of \citep{noisy-power-method} the largest principal angle between $\mat X_\ell$ and $\mat U_k$ is considered for every iteration $\ell$.
However, such analysis cannot possibly remove the dependency over $(\sigma_k-\sigma_{k-1})$, as we discuss in Sec.~\ref{sec:proof_sketch}.
To overcome such difficulties, we propose in Eq.~(\ref{eq_rank_perturb}) a novel characterization between a rank-$p$ subspace $\mat X_\ell$ and the rank-$k$ target space $\mat U_k$ through an intermediate subspace $\mathbf{U}_q$,
which we name as \emph{rank-$k$ perturbation on $\mat{U}_q$ by $\mathbf{X}_\ell$}.
This quantity does not correspond to any principal angle between linear subspaces when $p>k$.
Built upon the shrinkage behavior of the proposed quantity across iterations, we are able to obtain improved gap dependency for the noisy power method.
We hope our proof could shed light to the analysis of an even broader family of numerical linear algebra algorithms that involve noisy power iterations.

\subsection{Setup}
For a $d\times d$ positive semi-definite matrix $\mat A$, we denote $\mat A=\mat U\mat\Sigma\mat U^\top$ as its eigen-decomposition,
where $\mat U$ is an orthogonal $d\times d$ matrix and $\mat\Sigma=\diag(\sigma_1,\cdots,\sigma_d)$ is a $d\times d$ diagonal matrix consisting eigenvalues of $\mat A$,
sorted in descending order: $\sigma_1\geq\sigma_2\geq\cdots\geq\sigma_d\geq 0$.
The spectral norm $\|\mat A\|_2$ and Frobenious norm $\|\mat A\|_F$ can then be expressed as
$\|\mat A\|_2 = \sigma_1$ and $\|\mat A\|_F = \sqrt{\sigma_1^2+\cdots+\sigma_d^2}$.
For an integer $k\in[d]$ , we define $\mat U_k$ as a $d\times k$ matrix with orthonormal columns,
whose column space corresponds to the top-$k$ eigen-space of $\mat A$.
Similarly, $\mat\Sigma_k=\diag(\sigma_1,\cdots,\sigma_k)$ corresponds to the top-$k$ eigenvalues of $\mat A$.
Let $\mat A_k\in\argmin_{\mat B: \rank(\mat B)\leq k}\|\mat A-\mat B\|_\xi$ be the optimal rank-$k$ approximation of $\mat A$.
It is well-known that $\mat A_k=\mat U_k\mat\Sigma_k\mat U_k^\top$ is the optimal approximation for both spectral norm ($\xi=2$) and Frobenious norm ($\xi=F$)~\citep{eckart1936approximation}.

QR Factorization is a process to obtain an orthonormal column basis of a matrix.
For a $d \times p$ matrix $\mathbf{Y}$, QR factorization gives us $\mathbf{Y} = \mathbf{X}\mathbf{R}$ where $\mathbf{X} \in \mathbb{R}^{d \times p}$ is orthonormal and $\mathbf{R} \in \mathbb{R}^{p \times p}$ is an upper triangular matrix~\citep{trefethen1997numerical}.


\section{An improved analysis of the noisy power method}\label{sec:main}



\begin{algorithm}[h]
	\begin{algorithmic}
	\caption{The noisy matrix power method}
	\STATE{\textbf{Input}: positive semi-definite data matrix $\mat A\in\mathbb R^{d\times d}$, target rank $k$, iteration rank $p \geq k$, number of iterations $L$.}
	\STATE{\textbf{Output}: approximated eigen-space $\mat X_L\in\mathbb R^{d\times p}$, with orthonormal columns.}
	\STATE{\textbf{Initialization}}: orthonormal $\mat X_0\in\mathbb R^{d\times p}$ by QR decomposition on  random Gaussian matrix $\mat G_0$;\\
	\FOR{$\ell=1$ to $L$}
	\STATE{
		1. Observe $\mat Y_\ell = \mat A\mat X_{\ell-1} + \mat G_\ell$ for some noise matrix $\mat G_\ell$;}
	\STATE{
		2. QR factorization: $\mat Y_\ell = \mat X_\ell\mat R_\ell$, where $\mat X_\ell$ consists of orthonormal columns;}
	\ENDFOR
	\label{alg_noisy_power_method}
	\end{algorithmic}
\end{algorithm}

The noisy power method is described in Algorithm \ref{alg_noisy_power_method}.
\citep{noisy-power-method} provides the first general-purpose analysis of the convergence rate and noise tolerance of Algorithm \ref{alg_noisy_power_method}.
We cite their main theoretical result below:
\begin{thm}[\cite{noisy-power-method}]
	Fix $\epsilon\in(0,1/2)$ and let $k\leq p$. Let $\mat U_k\in\mathbb R^{d\times k}$ be the top-$k$ eigenvectors of a positive semi-definite matrix $\mat A$ and let $\sigma_1\geq\cdots\geq\sigma_n\geq 0$ denote
	its eigenvalues.
	Suppose at every iteration of the noisy power method the noise matrix $\mat G_\ell$ satisfies
	$$
	5\|\mat G_\ell\|_2 \leq \epsilon(\sigma_k-\sigma_{k+1}) \;\;\;\; \text{and} \;\;\;\; 5\|\mat U_k^\top\mat G_\ell\|_2 \leq (\sigma_k-\sigma_{k+1})\frac{\sqrt{p}-\sqrt{k-1}}{\tau\sqrt{d}}
	$$
	for some fixed constant $\tau$. Assume in addition that the number of iterations $L$ is lower bounded as
	$$
	L = \Omega\left(\frac{\sigma_k}{\sigma_k-\sigma_{k+1}}\log\left(\frac{d\tau}{\epsilon}\right)\right).
	$$
	Then with probability at least $1-\tau^{-\Omega(p+1-k)}-e^{-\Omega(d)}$ we have $\|(\mat I-\mat X_L\mat X_L^\top)\mat U_k\|_2\leq \epsilon$.
	\label{thm_main_before}
\end{thm}

Theorem \ref{thm_main_before} has one major drawback: both bounds for noise tolerance and convergence rate depend crucially on the ``small" singular value gap $(\sigma_k-\sigma_{k+1})$.
This gap could be extremely small for most data matrices in practice since it concerns the difference between two \emph{consecutive} singular values.
We show in later paragraphs an example where such gap-dependency could lead to significant deterioration in terms of both error tolerance and computing.
A perhaps even more disappointing fact is that the dependency over $(\sigma_k-\sigma_{k+1})$ cannot be improved under the existing analytical framework by increasing $p$, the number of 
components maintained by $\mat X_\ell$ at each iteration.
On the other hand, one expects the noisy power method to be more robust to per-iteration noise when $p$ is much larger than $k$.
This intuition has been formally established in \citep{mgu-subspace-iteration} under the noiseless setting and was also articulated as a conjecture in \citep{noisy-power-method}:
\begin{conj}[\cite{noisy-power-method}]
	The noise tolerance terms in Theorem \ref{thm_main_before} can be improved to 
	\begin{equation}\label{eq_conj1}
	5\|\mat G_\ell\|_2 \leq \epsilon(\sigma_k-\sigma_{p+1}) \;\;\;\; \text{and} \;\;\;\; 5\|\mat U_k^\top\mat G_\ell\|_2 \leq \frac{\sqrt{p}-\sqrt{k-1}}{\tau\sqrt{d}}.
	\end{equation}
	\label{conj_gap}
\end{conj}
In this section, we provide a more refined theoretical analysis of the noisy matrix power method presented in Algorithm \ref{alg_noisy_power_method}.
Our analysis significantly improves the gap dependency over existing results in Theorem \ref{thm_main_before} and partially solves Conjecture \ref{conj_gap} 
up to additional constant-level dependencies:
\begin{thm}[Improved gap-dependent bounds for noisy power method]\label{thm:new_main}
	Let $k\leq q \leq p$. Let $\mat U_q\in\mathbb R^{d\times q}$ be the 
	top-$q$ eigenvectors of a positive semi-definite matrix $\mat A$ and let $\sigma_1\geq\cdots\geq\sigma_d\geq 0$ denote
	its eigenvalues and fix $\epsilon = O\left(\frac{\sigma_q}{\sigma_k}\cdot\min\left\{\frac{1}{\log\left(\frac{\sigma_k}{\sigma_q}\right)},\frac{1}{\log\left(\tau d\right)}\right\}\right)$.
	Suppose at every iteration of the noisy power method the noise matrix $\mat G_\ell$ satisfies
	\begin{equation*}
	\|\mat G_\ell\|_2 = O\left(\epsilon(\sigma_k-\sigma_{q+1})\right) \quad \text{and} \quad
	\|\mat U_q^\top\mat G_\ell\|_2 = O\left(\epsilon\left(\sigma_k - \sigma_{q+1}\right)\frac{\sqrt{p}-\sqrt{q-1}}{\tau\sqrt{d}}\right)
	\end{equation*}
	for some constant $\tau>0$.
	Then after
	$$
	L = \Theta\left(\frac{\sigma_k}{\sigma_k-\sigma_{q+1}}\log\left(\frac{\tau d}{\epsilon }\right)\right).
	$$ iterations, with probability at least $1-\tau^{-\Omega(p+1-q)}-e^{-\Omega(d)}$, we have
	$$
	\|(\mat I-\mat X_L\mat X_L^\top)\mat U_k\|_2\leq\epsilon.
	$$
	Furthermore, for $\xi = 2$ or $F$, the low-rank approximation error $\norm{\mat{A}-\mat{X}_L\mat{X}_L^\top\mat{A}}_\xi$ is upper bounded as
	\begin{align*}
	\norm{\mat{A}-\mat{X}_L\mat{X}_L^\top\mat{A}}_\xi \le \left(1+\epsilon\right)\norm{\mat{A}-\mat{A}_k}_\xi,
	\end{align*} where $\mat A_k$ is the optimal rank-$k$ approximation of $\mat A$.	
\end{thm}
\paragraph{Discussion} 
Compared to existing bounds in Theorem \ref{thm_main_before}, the noise tolerance as well as convergence rate of noisy power method is significantly improved in 
Theorem \ref{thm:new_main}, where the main gap-dependent term $(\sigma_k-\sigma_{k+1})$ is improved to $(\sigma_k-\sigma_{q+1})$
for some intermediate singular value $\sigma_q$ with $k\leq q\leq p$.
Since the singular values are non-increasing, setting a large value of $q$ in Theorem \ref{thm:new_main} would improve the bounds.
However, $q$ cannot be too close to $p$ due to the presence of a $(\sqrt{p}-\sqrt{q-1})$ term.
In addition, the convergence rate (i.e., bound on $L$) specified in Theorem \ref{thm:new_main} reproduces recent results in \citep{mgu-subspace-iteration}
for noisy power method under noiseless settings ($\mat G_\ell=\mat 0$).
There are three main differences between our theorems and the conjecture raised by~\citep{noisy-power-method}.
First, the strength of projected noise $\mathbf{U}_q^\top\mathbf{G}$ also depends on $\epsilon$.
However, in many applications, this assumption is implied by the $\|\mat G_\ell\|_2 = O\left(\epsilon(\sigma_k-\sigma_{q+1})\right)$ assumption.
Second, we have $\left(\sqrt{p}-\sqrt{q-1}\right)$ instead of $\left(\sqrt{p}-\sqrt{k-1}\right)$ dependence. 
When $q=\Theta\left(k\right)$ and $q \ge 2p$, then this term is the at the same order as in the conjecture. Lastly, we notice that the second term of \eqref{eq_conj1} is totally independent of $\sigma_k, \sigma_{p+1}$ and their gap, which seems to be either a typo or unattainable result. Nonetheless, Theorem \ref{thm:new_main} has shown significant improvement on Theorem \ref{thm_main_before}.

To further shed light on the nature of our obtained results, we consider the following example 
to get a more interpretable comparison between Theorem~\ref{thm:new_main} and \ref{thm_main_before}:
\paragraph{Example: power-law decaying spectrum} 
We consider the example where the spectrum of the input data matrix $\mat A$ has \emph{power-law} decay;
that is, $\sigma_k \asymp k^{-\alpha}$ for some parameter $\alpha > 1$.
Many data matrices that arise in practical data applications have such spectral decay property \citep{dp-factorization}.
The small eigengap $(\sigma_k-\sigma_{k+1})$ is on the order of $k^{-\alpha-1}$.
As a result, the number of iterations $L$ should be at least $\Omega(k\log(d/\epsilon))$, 
which implies a total running time of $O(dk^3\log(d/\epsilon))$.
On the other hand, by setting $q=ck$ for some constant $c>1$ the ``large" spectral gap $(\sigma_k-\sigma_{q+1})$ is on the order of $k^{-\alpha}$.
Consequently, the number of iterations $L$ under the new theoretical analysis only needs to scale as $\Omega(\log(d/\epsilon))$
and the total number of flops is $O(dk^2\log(d/\epsilon))$.
This is an $O(k)$ improvement over existing bounds for noisy power method.

Apart from convergence rates, our new analysis also improves the noise tolerance (i.e., bounds on $\|\mat G_\ell\|_2$) in an explicit way when the data matrix $\mat A$
is assumed to have power-law spectral decay.
More specifically, old results in \citep{noisy-power-method} requires the magnitude of the noise matrix $\|\mat G_\ell\|_2$ to be upper bounded by $O(\epsilon k^{-\alpha-1})$,
while under the new analysis (Theorem \ref{thm:new_main}) a bound of the form $\|\mat G_\ell\|_2=O(\epsilon k^{-\alpha})$ suffices,
provided that $q=ck$ for some constant $c>1$ and $\epsilon$ is small.
This is another $O(k)$ improvement in terms of bounds on the maximum tolerable amount of per-iteration noise.

\subsection{Proof of Theorem \ref{thm:new_main}}\label{sec:proof_sketch}
Before presenting our proof of the main theorem (Theorem~\ref{thm:new_main}), we first review the arguments in \citep{noisy-power-method}
and explain why straightforward adaptations of their analysis cannot lead to improved gap dependency.
\citep{noisy-power-method} considered the tangent of the $k$th principle angle between $\mat U_k$ and $\mat X_\ell$:
\begin{equation}
\tan\theta_k(\mat U_k,\mat X_\ell) 
= \left\|(\mat U_{d-k}^\top\mat X_\ell)(\mat U_k^\top\mat X_\ell)^{\dagger}\right\|_2,
\label{eq_tank}
\end{equation}
where $\mat U_{d-k}\in\mathbb R^{d\times (d-k)}$ is the orthogonal complement of the top-$k$ eigen-space $\mat U_k\in\mathbb R^{d\times k}$ of $\mat A$.
It can then be shown that when both $\|\mat G_\ell\|_2$ and $\|\mat U_k^\top\mat G_\ell\|_2$ are properly bounded,
the angle geometrically shrinks after each power iteration; that is, $\tan\theta_k(\mat U_k,\mat X_{\ell+1}) \leq \rho\tan\theta_k(\mat U_k,\mat X_{\ell})$ for some fixed $\rho\in(0,1)$.
However, as pointed outed by~\citep{noisy-power-method}, this geometric shrinkage might not hold with larger level of noise.

To overcome such difficulties, in our analysis we consider a different characterization between $\mat U_k$ (or $\mat U_q$) and $\mat X_\ell$ at each iteration.
Let $\mat U_k\in\mathbb R^{d\times k}$, $\mat U_q\in\mathbb R^{d\times q}$ be the top $k$ and top $q$ eigenvectors of $\mat X$
and let $\mat U_{d-q}\in\mathbb R^{d\times(d-p)}$ be the remaining eigenvectors.
For an orthonormal matrix $\mat X_\ell\in\mathbb R^{d\times p}$, define the \emph{rank-$k$ perturbation on $\mat U_q$ by $\mat X_\ell$} as
\begin{equation}
h_\ell := \left\|(\mat U_{d-q}^\top\mat X_\ell)(\mat U_q^\top\mat X_\ell)^{\dagger}\left(\begin{array}{c} \mat I_{k\times k}\\ \mat 0\end{array}\right)\right\|_2.
\label{eq_rank_perturb}
\end{equation}
Intuitively, $h_\ell$ extracts a certain rank-$k$ component from $\tan\theta_p(\mat U_q,\mat X_\ell)=\|(\mat U_{d-q}^\top\mat X_\ell)(\mat U_q^\top\mat X_\ell)^{\dagger}\|_2$.
Consider the case when $p=q$, then ideally, $\mat{X}_\ell = \mathbf{U}_q$ and $(\mat U_q^\top\mat X_\ell)^{\dagger}$ is the identity matrix.
Here we relieve this goal that we only test whether the first $k$ columns of $(\mat U_q^\top\mat X_\ell)^{\dagger}$ is close to $\left(\begin{array}{c} \mat I_{k\times k}\\ \mat 0\end{array}\right)$.
It is also different from $\tan\theta_k(\mat U_k,\mat X_\ell)$ in Eq.~(\ref{eq_tank}), as the definition of $h_\ell$ involves both $\mat U_k$ and $\mat U_q$.
We can then show the following shrinkage results for $h_\ell$ across iterations:
\begin{lem}\label{thm:h_update_new}
	If the noise matrix at each iteration satisfies
	\begin{align*}
	\norm{\mat{G}_\ell}_2 &\le c \epsilon\left(\sigma_k - \sigma_{q+1}\right), \quad
	\norm{\mat{U}_q^\top\mat{G}_\ell}_2 \le c \cdot \min\{\epsilon\left(\sigma_k - \sigma_{q+1}\right)\cos\theta_q(\mat U_q,\mat X_\ell), \sigma_q\cos\theta_q(\mat U_q,\mat X_\ell)\},
	\end{align*} for some sufficiently small absolute constant $0< c < 1$, define
	$$\rho := \frac{\sigma_{q+1}+C\epsilon\left(\sigma_k - \sigma_{q+1}\right)}{\sigma_k}.$$
	we then have \begin{align*}
	h_{\ell+1} - \frac{C\epsilon\left(\sigma_k - \sigma_{q+1}\right)}{\left(1-\rho\right)\sigma_k} \le\rho\left(h_{\ell} - \frac{C\epsilon\left(\sigma_k - \sigma_{q+1}\right)}{\left(1-\rho\right)\sigma_k} \right),
	\end{align*} for some sufficiently small global constant $0 < C < 1$.
\end{lem}

The following lemma bounds the rank-$k$ perturbation on $\mat U_q$ by $\mat X_0$ when it is initialized via QR decomposition on a random Gaussian matrix $\mat G_0$,
as described in Algorithm \ref{alg_noisy_power_method}.
\begin{lem}\label{thm:h_initialization} With all but $\tau^{-\Omega\left(p+1-q\right)} + e^{-\Omega\left(d\right)}$ probability, we have that
	$$
	h_0 \leq \tan\theta_q(\mat U_q, \mat X_0) \leq \frac{\tau\sqrt{d}}{\sqrt{p}-\sqrt{q-1}}.
	$$
\end{lem}

Finally, Lemma \ref{lem:h_good} shows that small $h_L$ values imply small angles between $\mat X_L$ and $\mat U_k$.
\begin{lem}
	For any $\epsilon\in(0,1)$,
	if $h_L\leq\epsilon/4$ then $\tan\theta_k(\mat U_k,\mat X_L)\leq \epsilon$.
	\label{lem:h_good}
\end{lem}

The proofs of Lemma~\ref{thm:h_update_new}, \ref{thm:h_initialization} and \ref{lem:h_good} involve some fairly technical matrix computations and is thus deferred to Appendix~\ref{appsec:proof}. 
We are now ready to prove Theorem \ref{thm:new_main}:

\begin{proof}[\textbf{Theorem} \ref{thm:new_main}]
	First, the chosen $\epsilon$ ensures Corollary \ref{thm:angle_preserving} in Appendix~\ref{appsec:proof} holds, therefore, the noise conditions in Theorem \ref{thm:new_main} imply those noise conditions in Lemma \ref{thm:h_update_new} with high probability.
	As a result, the following holds for all $\ell\in[L]$:
	\begin{equation}
	h_{\ell+1} - \frac{C\epsilon\left(\sigma_k - \sigma_{q+1}\right)}{\left(1-\rho\right)\sigma_k} \le\rho\left(h_{\ell} - \frac{C\epsilon\left(\sigma_k - \sigma_{q+1}\right)}{\left(1-\rho\right)\sigma_k} \right),
	\label{eq_proof_hl}
	\end{equation}
	where $\rho=\frac{\sigma_{q+1}+C\epsilon(\sigma_k-\sigma_{q+1})}{\sigma_k}$ and $C$ is an absolute constant.
	Define $g_\ell := h_{\ell} -  \frac{C\epsilon\left(\sigma_k - \sigma_{q+1}\right)}{\left(1-\rho\right)\sigma_k}$.
	Eq.~(\ref{eq_proof_hl}) is then equivalent to $g_{\ell+1}\leq\rho g_\ell$.
	In addition, Lemma \ref{thm:h_initialization} yields
	$$
	g_0 \leq h_0 \leq \frac{\tau\sqrt{d}}{\sqrt{p}-\sqrt{q-1}}
	$$
	with high probability. Consequently, with $L=O(\log(g_0/\epsilon)/\log(1/\rho))$ iterations we have $g_L\leq\epsilon/2$.
	$h_L$ can then be bounded by 
	$$
	h_L = g_L+\frac{C\epsilon(\sigma_k-\sigma_{q+1})}{(1-\rho)\sigma_k} = \frac{\epsilon}{2} + \frac{C\epsilon\left(\sigma_k - \sigma_{q+1}\right)}{\sigma_k} \cdot\frac{\sigma_k}{\sigma_k - \sigma_{q+1}-C\epsilon\left(\sigma_k - \sigma_{q+1}\right)} \leq \epsilon.
	$$
	Subsequently, invoking Lemma \ref{lem:h_good} we get $\|(\mat I-\mat X_L\mat X_L^\top)\mat U_k\|_2 =\sin\theta_k(\mat U_k,\mat X_\ell)\leq \tan\theta_k(\mat U_k,\mat X_\ell) \le 8\epsilon=O(\epsilon)$, where we adopt the definition of $\sin\theta_k(\mat U_k,\mat X_\ell)$ from~\citep{noisy-power-method}.
	By Theorem 9.1 of~\citep{halko2011finding}, we can also obtain the desired bound on the residue norm $\|\mat A-\mat X_L\mat X_L^\top\mat A\|_\xi \leq (1+O(\epsilon))\|\mat A-\mat A_k\|_\xi$.
	The constant in $O(\epsilon)$ can be absorbed into the bounds of $\mat G_\ell$ and $L$.
	
	We next simplify the bound $L=O(\log(g_0/\epsilon)/\log(1/\rho))$. We first upper bound the shrinkage parameter $\rho$ as follows:
	\begin{align*}
	\rho  =& \frac{\sigma_{q+1}}{\sigma_k}+\frac{C\left(\sigma_k - \sigma_{q+1}\right)\epsilon}{\sigma_k} 
	\le  \frac{\sigma_{q+1} + \left(\sigma_k - \sigma_{q+1}\right)\epsilon/4}{\sigma_{q+1} + \left(\sigma_k - \sigma_{q+1}\right)/2} \\
	=& \frac{\sigma_{q+1} + \left(\sigma_{k}-\sigma_{q+1}\right)/4}{\sigma_{q+1} + \left(\sigma_{k}-\sigma_{q+1}\right)/2}\cdot \frac{\sigma_{q+1}}{\sigma_{q+1} + \left(\sigma_{k}-\sigma_{q+1}\right)/4}+
	\frac{ \left(\sigma_{k}-\sigma_{q+1}\right)/4}{\sigma_{q+1} + \left(\sigma_{k}-\sigma_{q+1}\right)/2} \cdot \epsilon\\
	\le & \max\left(\frac{\sigma_{q+1}}{\sigma_{q+1} + \left(\sigma_{k}-\sigma_{q+1}\right)/4},\epsilon\right),
	\end{align*}
	where the last inequality is due to that weighted mean is no larger than the maximum of two terms. Then we further have
	\begin{align*}
	\log(1/\rho)
	\ge & \log\left[\min\left( \frac{\sigma_{q+1}+(\sigma_k-\sigma_{q+1})/4}{\sigma_{q+1}}, {1\over\epsilon} \right)\right]
	\ge \min\left( \log{\sigma_{k}+3\sigma_{q+1}\over4\sigma_{q+1}}, 1 \right) \\
	\ge & \min \left( 1- {4\sigma_{q+1}\over \sigma_{k}+3\sigma_{q+1}}, 1  \right) = 1- {4\sigma_{q+1}\over \sigma_{k}+3\sigma_{q+1}}
	= {\sigma_k - \sigma_{q+1}\over \sigma_{k}+3\sigma_{q+1}}
	\end{align*}
	where the last inequality results from  $\log{\sigma_{k}+3\sigma_{q+1}\over4\sigma_{q+1}} \ge 1- {4\sigma_{q+1}\over \sigma_{k}+3\sigma_{q+1}}$.
	Subsequently, $\log(g_0/\epsilon)/\log(1/\rho)$ can be upper bounded as
	$$
	\frac{\log\left(g_0/\epsilon\right)}{\log\left(1/\rho\right)} 
	= O\left(\frac{\log\left(\tan\theta_q\left(\mat{U}_q,\mat{X}_0 \right)/\epsilon\right)}{(\sigma_k - \sigma_{q+1})/ (\sigma_{k}+3\sigma_{q+1})} \right)
	= O \left(\frac{\sigma_k}{\sigma_k - \sigma_{q+1}}\log\left(\frac{\tau d}{\epsilon}\right)\right),
	$$
	where we use the fact that $g_0 \le h_0 \le \tan\theta_q\left(\mat{U}_q,\mat{X}_0 \right)$ and the term $3\sigma_{q+1}$ is absorbed to $\sigma_k$.
\end{proof}

\subsection{Gap-independent bounds}

We lead a slight astray here to consider \emph{gap-independent} bounds for the noisy power method, which is a straightforward application of our derived 
gap-dependent bounds in Theorem \ref{thm:new_main}.
It is clear that the angle $\sin\theta_k(\mat U_k,\mat X_L)=\|(\mat I-\mat X_L\mat X_L^\top)\mat U_k\|_2$ cannot be gap-free,
because the top-$k$ eigen-space $\mat U_k$ is ill-defined when the spectral gap $(\sigma_k-\sigma_{k+1})$ or $(\sigma_k-\sigma_{q+1})$ is small.
On the other hand, it is possible to derive gap-independent bounds for the approximation error $\|\mat A-\mat X_L\mat X_L^\top\mat A\|_2$
because $\mat X_L$ does not need to be close to $\mat U_k$ to achieve good approximation of the original data matrix $\mat A$.
This motivates Hardt and Price to present the following conjecture on gap-independent bounds of noisy power method:
\begin{conj}[\cite{noisy-power-method}]\footnote{We rephrase the original conjecture to make $\epsilon$ not scale with singular values.}
	Fix $\epsilon\in (0,1)$, $p \ge 2k$ and suppose $\mat G_\ell$ satisfies
	\begin{equation}\label{eq_conj2}
	\|\mat G_\ell\|_2 = O(\epsilon\sigma_{k+1}), \quad \|\mat U_k^\top\mat G_\ell\|_2 = O\left(\epsilon\sigma_{k+1}\sqrt{k/d}\right)
	\end{equation}
	for all iterations $\ell=1,\cdots,L$.
	Then with high probability, after $L=O(\frac{\log d}{\epsilon})$ iterations we have
	$$
	\|\mat A-\mat X_L\mat X_L^\top\mat A\|_2 \leq (1+O(\epsilon))\|\mat A-\mat A_k\|_2 = (1+O(\epsilon))\sigma_{k+1}.
	$$
	\label{conj_gap_independent}
\end{conj}

Built upon the gap-dependent bound we derived in the previous section and a recent technique introduced in \citep{musco2015stronger} for the analysis of block Lanczos methods,
we are able to prove the following theorem that partially solves Conjecture \ref{conj_gap_independent}.
\begin{thm}\label{thm:gap_independent}
	Fix $0< \epsilon < 1$ and suppose the noise matrix satisfies \begin{align*}
	\norm{\mat{G}_\ell}_2 = O\left(\epsilon^2\sigma_{k+1}\right) \qquad \text{and} \qquad \norm{\mat{U}_k^\top \mat{G}_\ell}_2 = O \left( \frac{\epsilon^2\left(\sqrt{p}-\sqrt{k-1}\right)\sigma_{k+1}}{\tau\sqrt{d}} \right)
	\end{align*}for some constant $\tau>0$.
	Then after
	$$
	L = \Theta\left(\frac{1}{\epsilon}\log\left(\frac{\tau d}{\epsilon }\right)\right)
	$$ iterations,
	with probability at least $1-\tau^{-\Omega(p+1 - q)}-e^{-\Omega(d)}$, we have
	\begin{equation*}
	\norm{\mat{A} - \mat{X}_L\mat{X}_L^\top \mat{A}}_2\le \left(1+\epsilon\right)\norm{\mat{A}-\mat{A}_k}_2 = (1+\epsilon)\sigma_{k+1}.
	\end{equation*}
\end{thm}

The major difference between Theorem \ref{thm:gap_independent} and its targeted Conjecture \ref{conj_gap_independent} 
is an extra $O(\epsilon)$ term
in the noise bound of both $\|\mat G_\ell\|_2$ and $\|\mat U_k^\top\mat G_\ell\|_2$.
Whether such a gap can be closed remains an important open question.
The main idea of the proof is to find $m=\max_{0\leq i\leq k}\{\sigma_i-\sigma_{k+1}\geq\epsilon\sigma_{k+1}\}$
and apply Theorem \ref{thm:new_main} with $m$ as the new targeted rank and $k$ as the intermediate rank $q$.
A complete proof is deferred to Appendix \ref{appsec:gapfree}.
\section{Application to distributed private PCA}\label{sec:dppca}
Our main result can readily lead to improvement of several downstream applications, which will be highlighted in the this section and next. Specifically, we will discuss the benefit brought to distributed private PCA setting in this section, and memory-efficient streaming PCA in the next.

\subsection{The model}\label{subsec:dppca-model}

In our distributed private PCA model there are $s\geq 1$ computing nodes, each storing a positive semi-definite $d\times d$ matrix $\mat A^{(i)}$.
$\mat A^{(i)}$ can be viewed as the sample covariance matrix of data points stored on node $i$.
There is also a central computing node, with no data stored.
The objective is to approximately compute the top-$k$ eigen-space $\mat U_k$ of 
the aggregated data matrix $\mat A=\sum_{i=1}^s{\mat A^{(i)}}$
without leaking information of each data matrix $\mat A^{(1)},\cdots,\mat A^{(s)}$.
Each of the $s$ computing nodes can and only can communicate with the central node via a \emph{public} channel, where all bits communicated are public to the other nodes
as well as any malicious party.
We are interested in algorithms that meet the following formal guarantees:
\paragraph{Privacy guarantee} We adopt the concept of $(\varepsilon,\delta)$-differential privacy proposed in \citep{dwork2006our}.
Fix privacy parameters $\varepsilon,\delta\in(0,1)$.
Let $D$ be all bits communicated via the public channels between the $s$ computing nodes and the central node.
For every $i\in\{1,\cdots,s\}$ and all $\mat A^{(i)'}$ that differs from $\mat A^{(i)}$ in at most one entry with absolute difference at most 1, 
the following holds
\begin{equation}
\Pr\left[D\in\mathcal D|\mat A^{(i)},\mat A^{(-i)}\right] \leq e^{\varepsilon}\Pr\left[D\in\mathcal D|\mat A^{(i)'},\mat A^{(-i)}\right] + \delta,
\label{eq_privacy}
\end{equation}
where $\mat A^{(-i)}=(\mat A^{(1)},\cdots,\mat A^{(i-1)},\mat A^{(i+1)},\cdots,\mat A^{(s)})$ and $\mathcal D$ is any measurable set of $D$ bits communicated.

\paragraph{Utility guarantee} Suppose $\mat X_L$ is the $d\times p$ dimensional output matrix.
It is required that
$$
\sin\theta_k(\mat U_k, \mat X_L) = \|(\mat I-\mat X_L\mat X_L^\top)\mat U_k\|_2 \leq \epsilon
$$
with probability at least 0.9, where $\epsilon$ characterizes the error level
and $\mat U_k$ is the top-$k$ eigen-space of the aggregated data matrix $\mat A=\mat A^{(1)}+\cdots+\mat A^{(s)}$.

\paragraph{Communication guarantee} The total amount of bits communicated between the $s$ computing nodes and the central node is constrained.
More specifically, we assume only $M$ real numbers can be communicated via the public channels.

The model we considered is very general and reduces to several existing models of private or communication constrained PCA as special cases.
Below we give two such examples that were analyzed in prior literature.
\begin{rem}[Reduction from private PCA]
	Setting $s=1$ in our distributed private PCA model we obtain the private PCA model previously considered in \citep{noisy-power-method,private-pca-incoherent},
	\footnote{The $s=1$ case is actually harder than models considered in \citep{noisy-power-method,private-pca-incoherent} in that intermediate steps of noisy power method are released to the public
		as well. However this does not invalidate the analysis of noisy power method based private PCA algorithms because of the privacy composition rule.}
	where neighboring data matrices differ by one \emph{entry} with bounded absolute difference.
	\label{rem_private_pca}
\end{rem}

\begin{rem}[Reduction from distributed PCA]
	Setting $\varepsilon\to\infty$ and $\delta=0$ we obtain the distributed PCA model previously considered in \citep{distributed-pca},
	where columns (data points) are split and stored separately on different computing nodes.
	\label{rem_distributed_pca}
\end{rem}

\subsection{Algorithm and analysis}

We say an algorithm solves the $(\varepsilon,\delta,\epsilon,M)$-distributed private PCA problem if it satisfies all three guarantees mentioned in Sec.~\ref{subsec:dppca-model}
with corresponding parameters.
Algorithm \ref{alg_dppca} describes the idea of executing the noisy power method with Gaussian noise in a distributed manner.
\begin{algorithm}[h]
	\begin{algorithmic}
	\caption{Distributed private PCA via distributed noisy power method}
	\STATE{\textbf{Input}: distributedly stored data matrices $\mat A^{(1)},\cdots,\mat A^{(s)}\in\mathbb R^{d\times d}$, number of iterations $L$, target rank $k$, iteration rank $p\geq k$, 
		private parameters $\varepsilon,\delta$.}
	\STATE{\textbf{Output}: approximated eigen-space $\mat X_L\in\mathbb R^{d\times p}$, with orthonormal columns.}
	\STATE{\textbf{Initialization}}: orthonormal $\mat X_0\in\mathbb R^{d\times p}$ by QR decomposition on a random Gaussian matrix $\mat G_0$;
	noise variance parameter $\nu=4\varepsilon^{-1}\sqrt{pL\log(1/\delta)}$;\\
	\FOR{$\ell=1$ to $L$}
	\STATE{
		1. The central node broadcasts $\mat X_{\ell-1}$ to all $s$ computing nodes;\\
		2. Computing node $i$ computes $\mat Y^{(i)}_\ell=\mat A^{(i)}\mat X_{\ell-1}+\mat G^{(i)}_\ell$ with $\mat G^{(i)}_\ell\sim\nml(0,\|\mat X_{\ell-1}\|_{\infty}^2\nu^2)^{d\times p}$
		and sends $\mat Y^{(i)}_\ell$ back to the central node;\\
		3. The central node computes $\mat Y_\ell=\sum_{i=1}^s{\mat Y_\ell^{(i)}}$ and QR factorization $\mat Y_\ell=\mat X_\ell\mat R_\ell$.
	}
	\ENDFOR
	\end{algorithmic}
	\label{alg_dppca}
\end{algorithm}

The following theorem shows that Algorithm \ref{alg_dppca} solves the $(\varepsilon,\delta,\epsilon,M)$-distributed private PCA problem
with detailed characterization of the utility parameter $\epsilon$ and communication complexity $M$.
Its proof is deferred to Appendix \ref{appsec:dppca}.
\begin{thm}[Distributed private PCA]\label{thm:dppca}
	Let $s$ be the number of nodes and $\mat A^{(1)},\cdots,\mat A^{(s)}\in\mathbb R^{d\times d}$ be data matrices stored separately on the $s$ nodes.
	Fix target rank $k$, intermediate rank $q\geq k$ and iteration rank $p$ with $2q\leq p\leq d$.
	Suppose the number of iterations $L$ is set as $L=\Theta(\frac{\sigma_k}{\sigma_k-\sigma_{q+1}}\log(d))$.
	Let $\varepsilon,\delta\in(0,1)$ be privacy parameters.
	Then Algorithm \ref{alg_dppca} solves the $(\varepsilon,\delta,\epsilon,M)$-distributed PCA problem with
	$$
	\epsilon = O\left(\frac{\nu\sqrt{\mu(\mat A)s\log d\log L}}{\sigma_k-\sigma_{q+1}}\right)
	\quad\text{and}\quad
	M=O(spdL)=O\left(\frac{\sigma_k}{\sigma_k-\sigma_{q+1}}spd\log d\right).
	$$
	Here assuming conditions in Theorem~\ref{thm:new_main} are satisfied,  $\nu=\varepsilon^{-1}\sqrt{4pL\log(1/\delta)}$ and $\mu(\mat A)$ is the \emph{incoherence} \citep{private-pca-incoherent}
	of the aggregate data matrix $\mat A=\sum_{i=1}^s{\mat A^{(i)}}$;
	more specifically, $\mu(\mat A)=d\|\mat U\|_{\infty}$ where $\mat A=\mat U\mat\Lambda\mat U^{\top}$ is the eigen-decomposition of $\mat A$.
\end{thm}

It is somewhat difficult to evaluate the results obtained in Theorem \ref{thm:dppca} because our work, to our knowledge, is the first to consider distributed private PCA with the public channel communication model.
Nevertheless, on the two special cases of private PCA in Remark~\ref{rem_private_pca} and distributed PCA in Remark~\ref{rem_distributed_pca},
our result does significantly improve existing analysis.
More specifically, we have the following two corollaries based on Theorem \ref{thm:dppca} and Theorem \ref{thm:new_main}.
\begin{cor}[Improved private PCA]
	For the case of $s=1$ and $2p\leq q\leq d$, Algorithm \ref{alg_dppca} is $(\varepsilon,\delta)$-differentially private and $\mat X_L$ satisfies
	$$
	\|(\mat I-\mat X_L\mat X_L^\top)\mat U_k\|_2 \leq \epsilon = O\left(\frac{\nu\sqrt{\mu(\mat A)\log d\log L}}{\sigma_k-\sigma_{q+1}}\right)
	$$
	with probability at least 0.9. Here $\mat U_k$ is the top-$k$ eigen-space of input data matrix $\mat A\in\mathbb R^{d\times d}$.
	\label{cor_private_pca}
\end{cor}
\begin{cor}[Improved distributed PCA]
	Fix error tolerance parameter $\epsilon\in(0,1)$ and
	set $\nu=0$, $L=\Theta(\frac{\sigma_k}{\sigma_k-\sigma_{q+1}}\log(d/\epsilon))$ in Algorithm \ref{alg_dppca}. 
	We then have with high probability,
	$$
	\|(\mat I-\mat X_L\mat X_L^\top)\mat U_k\|_2 \leq \epsilon.
	$$
	Here $\mat U_k$ is the top-$k$ eigen-space of the aggregated matrix $\mat A=\sum_{i=1}^{s}{\mat A^{(i)}}$.
	\label{cor_distributed_pca}
\end{cor}
The proofs of Corollary \ref{cor_private_pca} and \ref{cor_distributed_pca} are simple and deferred to Appendix \ref{appsec:dppca}.
We now compare them with existing results in the literature.
For private PCA, our bound has better spectral-gap dependency compared to the $O(\frac{\nu\sqrt{\mu(\mat A)\log d\log L}}{\sigma_k-\sigma_{k-1}})$ bound
obtained in \citep{noisy-power-method}.
For distributed PCA, our bound achieves an \emph{exponential} improvement over the $O(spd/\epsilon)$ communication complexity bound obtained in \citep{distributed-pca}.
\footnote{
	Lemma 8 of~\cite{distributed-pca} gives a communication upper bound that depends on all singular values bigger than $k$.
	It is not obvious which bound is better, but in the worst case, their bound is still linear in $\frac{1}{\epsilon}$.
}

\section{Application to memory-efficient streaming PCA}

In the streaming PCA setting a computing machine receives a stream of samples $\vct{z}_1,\cdots\vct{z}_n \in \mathbb{R}^d$ drawn i.i.d~from an unknown underlying distribution $\mathcal{D}$. 
The objective is to compute the leading $k$ eigenvectors of the population covariance matrix $\mathbf{E}_{\vct{z}\sim \mathcal{D}}[\vct z\vct z^\top]$ 
with memory space constrained to the output size $O(kd)$.
\citep{mitliagkas2013memory} gave an algorithm for this problem based on the noisy power method. Algorithm~\ref{algo:streaming} gives the details. 

\begin{algorithm}[h]
	\begin{algorithmic}
	\caption{Memory-efficient Streaming PCA \citep{mitliagkas2013memory}}
	\STATE{\textbf{Input}: data stream $\vct z_1,\cdots,\vct z_n\overset{i.i.d.}{\sim}\mathcal D$, target rank $k$, iteration rank $p \geq k$, iteration number $L$.}
	\STATE{\textbf{Output}: approximated eigen-space $\mat X_L\in\mathbb R^{d\times p}$, with orthonormal columns.}
	\STATE{\textbf{Initialization}: uniformly sampled orthonormal matrix $\mat X_0\in\mathbb R^{d\times p}$; $T = \lfloor n/L\rfloor$.}
	\FOR{$\ell=1$ to $L$}
	\STATE{
		1. \textbf{Power update}: $\mat{Y}_\ell = \mat{A}_\ell\mat{X}_{\ell-1}$, where $\mat{A}_\ell = \sum_{i=\left(\ell-1\right)T+1}^{\ell T}\vct{z}_i\vct{z}_i^\top$;\\
		2. \textbf{QR factorization}: $\mat Y_\ell = \mat X_\ell\mat R_\ell$, where $\mat X_\ell$ consists of orthonormal columns.
	}
	\ENDFOR
	\label{algo:streaming}
	\end{algorithmic}
\end{algorithm}

\citep{noisy-power-method} are among the first ones that analyze Algorithm~\ref{algo:streaming} for a broad class of distributions $\mathcal D$
based on their analysis of the noisy power method.
More specifically, \citep{noisy-power-method} analyzed a family of distributions that have fast tail decay and proved gap-dependent sample complexity bounds for the 
memory-efficient streaming PCA algorithm.
\begin{defn}[$(B,p)$-round distributions, \citep{noisy-power-method}]
	A distribution $\mathcal{D}$ over $\mathbb{R}^d$ is $\left(B,p\right)-round$ if for every $p$-dimension projection $\mat{\Pi}$ and all $t\ge 1$, we have that
	$$
	\max\left\{\Pr_{\vct z\sim\mathcal D}\left[\|\vct z\|_2\geq t\right], \Pr_{\vct z\sim\mathcal D}\left[\|\mat\Pi\vct z\|_2\geq t\sqrt{Bp/d}\right]\right\} \leq \exp(-t).
	$$
\end{defn}

\begin{thm}[\citep{noisy-power-method}]
	Suppose $\mathcal D$ is a $(B,p)$-round distribution over $\mathbb R^d$. 
	Let $\sigma_1\geq\cdots\geq\sigma_d\geq 0$ be the singular values of the population covariance matrix $\mathbb E_{\vct z\sim\mathcal D}[\vct z\vct z^\top]$.
	If Algorithm \ref{algo:streaming} is run with $L=\Theta(\frac{\sigma_k}{\sigma_k-\sigma_{k+1}}\log(d/\epsilon))$ and $n$ satisfies
	\footnote{In the $\widetilde\Omega(\cdot)$ notation we omit poly-logarithmic terms.}
	$$
	n = \widetilde\Omega\left(\frac{\sigma_k B^2p\log^2d}{(\sigma_k-\sigma_{k+1})^3d\epsilon^2}\right),
	$$
	then with probability at least 0.9 we have that $\|(\mat I-\mat X_L\mat X_L^\top)\mat U_k\|_2\leq\epsilon$,
	where $\mat U_k$ is the top-$k$ eigen-space of $\mathbb E_{\vct z\sim\mathcal D}[\vct z\vct z^\top]$.
	\label{thm_streaming_pca_original}
\end{thm}

Recently, \citep{li2015rivalry} proposed a modified power method that achieves a logarithmic sample complexity improvement with respect to $1/\epsilon$.
Nevertheless, both bounds in \citep{noisy-power-method} and \citep{li2015rivalry} depend on the consecutive spectral gap $(\sigma_k - \sigma_{k+1})$,
which could be very small for real-world data distributions.
Built upon our analysis for the noisy power method, we obtain the following result for streaming PCA with improved gap dependencies:
\begin{thm}
	Fix $k\leq q\leq p\leq d$.
	Suppose $\mathcal D$ is a $(B,p)$-round distribution over $\mathbb R^d$. 
	Let $\sigma_1\geq\cdots\geq\sigma_d\geq 0$ be the singular values of the population covariance matrix $\mathbb E_{\vct z\sim\mathcal D}[\vct z\vct z^\top]$.
	If Algorithm \ref{algo:streaming} is run with $L=\Theta(\frac{\sigma_k}{\sigma_k-\sigma_{q+1}}\log(d/\epsilon))$ and $n$ satisfies
	$$
	n = \widetilde\Omega\left(\frac{\sigma_k B^2p\log^2d}{(\sigma_k-\sigma_{q+1})^3d\epsilon^2}\right),
	$$
	then with probability at least 0.9 we have that $\|(\mat I-\mat X_L\mat X_L^\top)\mat U_k\|_2\leq\epsilon$.
	\label{thm_streaming_pca_improve}
\end{thm}
\begin{proof}
	Note that Algorithm \ref{algo:streaming} is a direct application of noisy power method with
	$\mat{G}_\ell = \left(\mat{A}-\mat{A}_\ell\right)\mat{X}_{\ell-1}$, where $\mat A=\mathbb E_{\vct z\sim\mathcal D}[\vct z\vct z^\top]$
	is the covariance matrix of the population distribution of interest.
	By Lemma 3.5 of~\citep{noisy-power-method}, we have that 
	\begin{align*}
	T = \widetilde{\Omega}\left(\frac{B^2p\log\left(d\right)}{\epsilon^2\left(\sigma_k - \sigma_{q+1}\right)^2}\right),
	\end{align*} 
	is sufficient to guarantee that $\mat{G}_\ell$ satisfy the conditions in Theorem~\ref{thm:new_main} with high probability. 
	Therefore, in total we need $n=LT=\widetilde\Omega(\frac{\sigma_k B^2p\log^2d}{(\sigma_k-\sigma_{q+1})^3d\epsilon^2})$ data points.
\end{proof}

\section{Conclusions and Future Work}
In this paper we give a novel analysis of spectral gap dependency for noisy power method, which partially solves a conjecture raised in~\citep{noisy-power-method} with additional mild conditions.
As a by product, we derive a spectral gap independent bound which partially solved another conjecture in~\citep{noisy-power-method}.
Furthermore, our analysis directly leads to improved utility guarantees and sample complexity for downstream applications such as distributed PCA, private PCA and streaming PCA problems.

To completely solve the two conjectures in~\citep{noisy-power-method}, we need a finer robustness analysis of $\mat{U}_{p-k}$ space.
\citep{wang2015improved} gave a related analysis, but only for the noiseless case.
Potentially, we may define a new function (like Eq.~\eqref{eq_rank_perturb} in our case) to characterize the convergence behavior, and show it shrinks multiplicatively at each iteration.

In parallel to power method based algorithms, Krylov iteration is another method shown to converge faster in the noiseless case~\citep{musco2015stronger}.
It is also interesting to give a noise tolerance analysis for Krylov iteration and apply it to downstream applications.

\bibliography{improvenpm_arxiv}
\bibliographystyle{abbrv}

\newpage
\appendix

\section{Proofs of technical lemmas in Sec.~\ref{sec:proof_sketch}}\label{appsec:proof}

\begin{lem}[Lemma \ref{thm:h_update_new}]
	If the noise matrix at each iteration satisfies
	\begin{align*}
	\norm{\mat{G}_\ell}_2 &\le c \epsilon\left(\sigma_k - \sigma_{q+1}\right), \quad
	\norm{\mat{U}_q^\top\mat{G}_\ell}_2 \le c \cdot \min\{\epsilon\left(\sigma_k - \sigma_{q+1}\right)\cos\theta_q(\mat U_q,\mat X_\ell), \sigma_q\cos\theta_q(\mat U_q,\mat X_\ell)\},
	\end{align*} for some sufficiently small absolute constant $0< c < 1$, define
	$$\rho := \frac{\sigma_{q+1}+C\epsilon\left(\sigma_k - \sigma_{q+1}\right)}{\sigma_k}.$$
	we then have \begin{align*}
	h_{\ell+1} - \frac{C\epsilon\left(\sigma_k - \sigma_{q+1}\right)}{\left(1-\rho\right)\sigma_k} \le\rho\left(h_{\ell} - \frac{C\epsilon\left(\sigma_k - \sigma_{q+1}\right)}{\left(1-\rho\right)\sigma_k} \right),
	\end{align*} for some sufficiently small global constant $0 < C < 1$.
\end{lem}

\begin{proof}
	First notice that \begin{align*}
	\mat{U}_q^\top	\left(\mat{A}\mat{X}_\ell+ \mat{G}_\ell\right)\mat{R}_{\ell+1}^{-1}\left(\mat{R}_{\ell+1}\left(\mat{U}_q^\top\left(\mat{A}\mat{X}_\ell+\mat{G}_\ell\right)\right)^{\dagger}\right) = \mat{I}_{q \times q}.
	\end{align*}
	Therefore, the pseudo-inverse of $\mat{U}_q^\top\left(\mat{A}\mat{X}_\ell+ \mat{G}_\ell\right)\mat{R}_{\ell+1}^{-1}$ is $\mat{R}_{\ell+1}\left(\mat{U}_q^\top\left(\mat{A}\mat{X}_\ell+\mat{G}_\ell\right)\right)^{\dagger}$. We can then write out $h_{\ell + 1}$ explicitly:
	\begin{align*}
	h_{\ell+1} & = \norm{\mat{U}_{d-q}^\top\mat{X}_{\ell+1}\left(\mat{U}_q^\top\mat{X}_{\ell+1}\right)^{\dagger} \begin{pmatrix}
		\mat{I}_{k \times k} \\
		\mat{0}
		\end{pmatrix}}_2 \\
	& = \norm{
		\mat{U}_{d-q}^\top 
		\left(\mat{A}\mat{X}_\ell+ \mat{G}_\ell\right)
		\mat{R}_{\ell+1}^{-1}
		\left(\mat{U}_{q}^\top
		\left(\mat{A}\mat{X}_\ell + \mat{G}_\ell\right)\mat{R}_{\ell+1}^{-1}
		\right)^{\dagger}
		\begin{pmatrix}
		\mat{I}_{k \times k}\\
		\mat{0}
		\end{pmatrix}		
	}_2\\
	& = \norm{
		\mat{U}_{d-q}^\top 
		\left(\mat{A}\mat{X}_\ell+ \mat{G}_\ell\right)
		\left(\mat{U}_{q}^\top
		\left(\mat{A}\mat{X}_\ell + \mat{G}_\ell\right)
		\right)^{\dagger}
		\begin{pmatrix}
		\mat{I}_{k \times k}\\
		\mat{0}
		\end{pmatrix}		
	}_2 \\
	& = \norm{
		\left(
		\mat{\Sigma}_{d-q}\mat{U}_{d-q}^\top\mat{X}_\ell + \mat{U}_{d-q}^\top\mat{G}_\ell
		\right)
		\left(
		\mat{\Sigma}_p \mat{U}_{q}^\top\mat{X}_\ell + \mat{U}_{q}\mat{G}_\ell
		\right)^{\dagger}
		\begin{pmatrix}
		\mat{I}_{k \times k}\\
		\mat{0}
		\end{pmatrix}
	}_2\\
	& = \norm{
		\left(
		\mat{\Sigma}_{d-q}\mat{U}_{d-q}^\top\mat{X}_\ell + \mat{U}_{d-q}^\top\mat{G}_\ell
		\right)
		\left(
		\mat{U}_{q}^\top\mat{X}_\ell + \mat{\Sigma}^{-1}_p \mat{U}_{q}\mat{G}_\ell
		\right)^{\dagger}
		\begin{pmatrix}
		\mat{\Sigma}_{k}\\
		\mat{0}
		\end{pmatrix}
	}_2.
	\end{align*} Now we focus on the pseudo-inverse in the above expression. Our analysis relies on the following singular value decomposition (SVD) of $\mat{U}_{q}^\top\mat{X}_\ell$:
	\begin{align*}
	\mat{U}_{q}^\top\mat{X}_\ell = \widetilde{\mat{U}}\widetilde{\mat \Sigma}\widetilde{\mat V}^\top \in \mat{R}^{p \times q}.
	\end{align*} For simplicity, define \begin{align*}
	\widetilde{\mat{P}} = \widetilde{\mat{U}}\widetilde{\mat{\Sigma}}.
	\end{align*} 
	Subsequently, we have that
	\begin{align*}
	\mat{U}_{q}^\top\mat{X}_\ell = \widetilde{\mat{P}} \widetilde{\mat{V}}^\top \qquad \text{and} \qquad \mat{X}_\ell^\top\mat{U}_{q}\widetilde{\mat{P}}^{-\top} = \widetilde{\mat{V}}.
	\end{align*}
	By definition of pseudo-inverse, we have 
	\begin{align*}
	&\left(\mat{U}_{q}^\top\mat{X}_\ell + \mat{\Sigma}_p^{-1}\mat{U}_{q}^\top\mat{G}_\ell\right)^{\dagger} \\
	=& \left(\mat{U}_{q}^\top\mat{X}_\ell + \mat{\Sigma}_p^{-1}\mat{U}_{q}^\top\mat{G}_\ell\right)^\top \left[\left(\mat{U}_{q}^\top\mat{X}_\ell + \mat{\Sigma}_p^{-1}\mat{U}_{q}^\top\mat{G}_\ell\right)\left(\mat{U}_{q}^\top\mat{X}_\ell + \mat{\Sigma}_p^{-1}\mat{U}_{q}^\top\mat{G}_\ell\right)^{\top}\right]^{-1}.
	\end{align*} The inversion in the above expression can be related to our assumptions of noise:
	\begin{align*}
	& \left[\left(\mat{U}_{q}^\top\mat{X}_\ell + \mat{\Sigma}_p^{-1}\mat{U}_{q}^\top\mat{G}_\ell\right)\left(\mat{U}_{q}^\top\mat{X}_\ell + \mat{\Sigma}_p^{-1}\mat{U}_{q}^\top\mat{G}_\ell\right)^{\top}\right]^{-1} \\
	= &\left[
	\left(\widetilde{\mat P}\widetilde{\mat{V}}^\top + \mat{\Sigma}_q \mat{U}_{q}^\top\mat{G}_\ell\right)
	\left(\widetilde{\mat{V}}\widetilde{\mat{P}}^\top + \mat{G}_\ell^\top\mat{U}_{q}\mat{\Sigma}_q^{-1}\right)
	\right]^{-1} \\
	= & \widetilde{\mat{P}}^{-\top}\left[
	\left(\widetilde{\mat{V}}^\top + \widetilde{\mat{P}}^{-1}\mat{\Sigma}_q^{-1}\mat{U}_{q}^\top\mat{G}_\ell\right)
	\left(\widetilde{\mat{V}} + \mat{G}_\ell^\top\mat{U}_{q} \mat{\Sigma}_q^{-1}\widetilde{\mat{P}}^{-\top}\right)
	\right]^{-1}\widetilde{\mat{P}}^{-1}\\
	= & \widetilde{\mat{P}}^{-\top}\left[ \mat{I} + \widetilde{\mat{V}}^\top\mat{G}_\ell^\top\mat{U}_{q} \mat{\Sigma}_q^{-1}\widetilde{\mat{P}}^{-\top} + \widetilde{\mat{P}}^{-1}\mat{\Sigma}_q^{-1}\mat{U}_{q}^\top\mat{G}_\ell\widetilde{\mat{V}}  + \widetilde{\mat{P}}^{-1}\mat{\Sigma}_q^{-1}\mat{U}_{q}^\top\mat{G}_\ell\mat{G}_\ell^\top\mat{U}_{q} \mat{\Sigma}_q^{-1}\widetilde{\mat{P}}^{-\top}
	\right]^{-1}\widetilde{\mat{P}}^{-1} \\
	= & \widetilde{\mat{P}}^{-\top} \left(\mat{I} - \left(\mat{I} + \mat{Y}\right)^{-1}\mat{Y}\right) \widetilde{\mat{P}}^{-1},
	\end{align*} where $\mat{Y} =  \widetilde{\mat{V}}^\top\mat{G}_\ell^\top\mat{U}_{q} \mat{\Sigma}_q^{-1}\widetilde{\mat{P}}^{-\top} + \widetilde{\mat{P}}^{-1}\mat{\Sigma}_q^{-1}\mat{U}_{q}^\top\mat{G}_\ell\widetilde{\mat{V}}  + \widetilde{\mat{P}}^{-1}\mat{\Sigma}_q^{-1}\mat{U}_{q}^\top\mat{G}_\ell\mat{G}_\ell^\top\mat{U}_{q} \mat{\Sigma}_q^{-1}\widetilde{\mat{P}}^{-\top}$ and the last equation is by Woodbury's identity. Based on our noise assumptions, we can bound $\mat{Y}$ as
	\begin{equation}
	\norm{\mat{Y}}_2  \le 2 \frac{\norm{\mat{U}_{q}^\top\mat{G}_\ell}_2}{\sigma_q \sigma_{\min}\left(\mat{U}_{q}^\top\mat{X}_\ell\right)} + \frac{\norm{\mat{U}_{q}^\top\mat{G}_\ell}_2^2}{\sigma_q^2 \sigma_{\min}^2\left(\mat{U}_{q}^\top\mat{X}_\ell\right)} \\
	\le c_1 \min\left\{\frac{\epsilon\left(\sigma_k - \sigma_{q+1}\right)}{\sigma_q},1\right\},
	\label{eq_y_bound_1}
	\end{equation}
	for some constant $0< c_1 <1$. 
	Subsequently, we have that
	\begin{equation}
	\norm{\left(\mat{I} + \mat{Y}\right)^{-1}\mat{Y}}_2 \le \frac{\norm{\mat{Y}}_2}{1-\norm{\mat{Y}}_2} \\
	\le c_2 \frac{\epsilon\left(\sigma_k - \sigma_{q+1}\right)}{\sigma_q},
	\label{eq_y_bound_2}
	\end{equation}
	for some constant $0 < c_2 < 1$. 
	Applying triangle inequality we obtain upper bounds on $h_{\ell+1}$:
	\begin{align*}
	h_{\ell+1} \le & \norm{\mat{\Sigma}_{d-q}\mat{U}_{d-q}^\top\mat{X}_\ell \left(\mat{U}_{q}^\top\mat{X}_\ell + \mat{\Sigma}_q^{-1}\mat{U}_{q}^\top\mat{G}_\ell\right)^{\dagger} \begin{pmatrix}
		\mat{\Sigma}_k^{-1}\\
		\mat{0}
		\end{pmatrix}}_2 \\
	+
	&\norm{\mat{U}_{d-q}\mat{G}_\ell \left(\mat{U}_{q}^\top\mat{X}_\ell + \mat{\Sigma}_q^{-1}\mat{U}_{q}^\top\mat{G}_\ell\right)^{\dagger} \begin{pmatrix}
		\mat{\Sigma}_k^{-1}\\
		\mat{0}
		\end{pmatrix}}_2.
	\end{align*}
	We next bound the two terms in the right-hand side of the above inequality separately.
	For the first term, we have that
	\begin{align*}
	&\norm{\mat{\Sigma}_{d-q}\mat{U}_{d-q}^\top\mat{X}_\ell \left(\mat{U}_{q}^\top\mat{X}_\ell + \mat{\Sigma}_q^{-1}\mat{U}_{q}^\top\mat{G}_\ell\right)^{\dagger} \begin{pmatrix}
		\mat{\Sigma}_k^{-1}\\
		\mat{0}
		\end{pmatrix}}_2 \\
	= & 
	\left\|\mat{\Sigma}_{d-q}\mat{U}_{d-q}^\top\mat{X}_\ell
	\left[ \left(\mat{U}_{q}^\top\mat{X}_\ell\right)^{\dagger} +
	\mat{G}_\ell^\top\mat{U}_{q}\mat{\Sigma}_q^{-1}\widetilde{\mat{P}}^{-\top}\widetilde{\mat{P}}^{-1} \right.\right.\\
	&\left.\left.	+
	\left(\mat{U}_{q}^\top\mat{X}_\ell\right)^\top\widetilde{\mat{P}}^{-\top}\left(\mat{I}+\mat{Y}\right)^{-1}\mat{Y}\widetilde{\mat{P}}^{-1} + 
	\mat{G}_\ell^\top\mat{U}_{q}\mat{\Sigma}_q\widetilde{\mat{P}}^{-\top}\left(\mat{I}+\mat{Y}\right)^{-1}\mat{Y}\widetilde{\mat{P}}^{-1} \begin{pmatrix}
	\mat{\Sigma}_k^{-1} \\
	\mat{0}
	\end{pmatrix}
	\right]	
	\right\|_2\\
	\le & \frac{1}{\sigma_k}\left(
	\sigma_{q+1}h_\ell + \frac{c_1\sigma_{q+1}\epsilon\left(\sigma_k - \sigma_{q+1}\right)}{\sigma_q}\left(1+h_\ell\right) + \frac{c_2\sigma_{q+1}\epsilon\left(\sigma_k - \sigma_{q+1}\right)}{\sigma_q}\left(1+h_\ell\right) \right.\\
	&\left. + 
	\frac{c_1\sigma_{q+1}\epsilon\left(\sigma_k - \sigma_{q+1}\right)}{\sigma_q}\frac{c_2\epsilon\left(\sigma_k - \sigma_{q+1}\right)}{\sigma_q}\left(1+h_\ell\right)
	\right) \\
	\le& \frac{\sigma_{q+1}+c_4\epsilon\left(\sigma_k - \sigma_{q+1}\right)}{\sigma_k}h_\ell + \frac{c_4\epsilon\left(\sigma_k - \sigma_{q+1}\right)}{\sigma_k},
	\end{align*} 
	for some constant $0< c_4 < 1$.
	Here the second inequality is due to Eq.~(\ref{eq_y_bound_1},\ref{eq_y_bound_2}) and Lemma~\ref{thm:modified_cos_tan_inequality},
	Similarly, for the second term related to $\mat{U}_{d-q}\mat{G}_\ell$ we have that \begin{equation*}
	\norm{\mat{U}_{d-q}\mat{G}_\ell \left(\mat{U}_{q}^\top\mat{X}_\ell + \mat{\Sigma}_q^{-1}\mat{U}_{q}^\top\mat{G}_\ell\right)^{\dagger} \begin{pmatrix}
		\mat{I}_{k \times k}\\
		\mat{0}
		\end{pmatrix}}_2 
	\le \frac{c_5\epsilon\left(\sigma_k - \sigma_{q+1}\right)}{\sigma_k}h_\ell + \frac{c_5\epsilon\left(\sigma_k - \sigma_{q+1}\right)}{\sigma_k},
	\end{equation*} 
	for some constant $0 < c_5 < 1$. Merging these two bounds we arrive at our desired result.
\end{proof}

\begin{lem}\label{thm:modified_cos_tan_inequality}
	\begin{align*}
	\norm{\widetilde{\mat{P}}^{-1} \begin{pmatrix}
		\mat{I}_{k \times k}\\
		\mat{0}
		\end{pmatrix}}_2 \le 1 + h_\ell.
	\end{align*}
\end{lem}

\begin{proof}
	\begin{align*}
	\norm{\widetilde{\mat{P}}^{-1}\begin{pmatrix}
		\mat{I}_{k \times k}\\
		\mat{0}
		\end{pmatrix}}_2
	&= \norm{
		\left(\widetilde{\mat{U}}\widetilde{\mat{\Sigma}}\right)^{-1}\begin{pmatrix}
		\mat{I}_{k \times k}\\
		\mat{0}
		\end{pmatrix}
	}_2
	= \norm{
		\widetilde{\mat{\Sigma}}^{-1}\widetilde{\mat{U}}^\top \begin{pmatrix}
		\mat{I}_{k \times k}\\
		\mat{0}
		\end{pmatrix}			
	}_2\\
	& = \norm{
		\widetilde{\mat{V}}^\top\widetilde{\mat{V}}\widetilde{\mat{\Sigma}}^{-1}\widetilde{\mat{U}}\begin{pmatrix}
		\mat{I}_{k \times k}\\
		\mat{0}
		\end{pmatrix}
	}_2
	\le   \norm{
		\widetilde{\mat{V}}\widetilde{\mat{\Sigma}}^{-1}\widetilde{\mat{U}}\begin{pmatrix}
		\mat{I}_{k \times k}\\
		\mat{0}
		\end{pmatrix}
	}_2 
	= \norm{
		\left(\mat{U}_{q}^\top\mat{X}_\ell^\top\right)^{\dagger} \begin{pmatrix}
		\mat{I}_{k \times k}\\
		\mat{0}
		\end{pmatrix}
	}_2\\
	&= \norm{\mat{X}_\ell^\top \mat{X}_\ell\left(\mathbf{U}_1^\top \mat{X}_l\right)^{\dagger}\begin{pmatrix}
		\mat{I}_{k \times k} \\
		\mat{0}
		\end{pmatrix}}_2 
	\le \norm{ \mat{X}_\ell\left(\mathbf{U}_1^\top \mat{X}_l\right)^{\dagger}\begin{pmatrix}
		\mat{I}_{k \times k} \\
		\mat{0}
		\end{pmatrix}}_2 \\
	&= \norm{\left(\mat{U}_q\mat{U}_q^\top + \mat{U}_{d-q} \mat{U}_{d-q}^\top\right)\mat{X}_\ell\left(\mathbf{U}_1^\top \mat{X}_l\right)^{\dagger}\begin{pmatrix}
		\mat{I}_{k \times k} \\
		\mat{0}
		\end{pmatrix}}_2 \\
	&\le 1 + \norm{\mat{U}^\top_2\mat{X}_\ell\left(\mathbf{U}_1^\top \mat{X}_\ell\right)^{-1}\begin{pmatrix}
		\mat{I}_{k \times k} \\
		\mat{0}
		\end{pmatrix}}_2
	=  1 + h_\ell.
	\end{align*}
\end{proof}

\begin{lem}[Lemma~\ref{thm:h_initialization}]
	With all but $\tau^{-\Omega\left(p+1-q\right)} + e^{-\Omega\left(d\right)}$ probability, we have thta
	$$
	h_0 \leq \tan\theta_q(\mat U_q, \mat X_0) \leq \frac{\tau\sqrt{d}}{\sqrt{p}-\sqrt{q-1}}.
	$$
\end{lem}
\begin{proof}
	Notice that $\mat{U}_{d-q}^\top\mat{X}_0\left(\mat{U}_q^\top\mat{X}_0\right)^{\dagger}\begin{pmatrix}
	\mat{I}_{k \times k}\\
	\mat{0}
	\end{pmatrix}$ 
	is a sub-matrix of $\mat{U}_{d-q}^\top\mat{X}_0\left(\mat{U}_q^\top\mat{X}_0\right)^\dagger$. Therefore,
	\begin{equation*}
	h_0 = \norm{\mat{U}_{d-q}^\top\mat{X}_0\left(\mat{U}_q^\top\mat{X}_0\right)^\dagger\begin{pmatrix}
		\mat{I}_{k \times k}\\
		\mat{0}
		\end{pmatrix}}_2
	\le \norm{\mat{U}_{d-q}^\top\mat{X}_0\left(\mat{U}_q^\top\mat{X}_0\right)^\dagger}_2
	= \tan\theta_q\left(\mat U_q, \mat{X}_0\right).
	\end{equation*}
	By $\mat X_0$ is the column space of a $d\times p$ random Gaussian matrix,
	Lemma 2.5 in \citep{noisy-power-method} yields
	\begin{align*}
	\tan\theta_q\left(\mat U_q, \mat{X}_0 \right) \le \frac{\tau\sqrt{d}}{\sqrt{p}-\sqrt{q-1}}
	\end{align*} with all but $\tau^{-\Omega\left(p+1-q\right)} + e^{-\Omega\left(d\right)}$ probability.
\end{proof}

\begin{lem}[Lemma \ref{lem:h_good}]
	If $h_L\leq\epsilon/4$ then $\tan\theta_k(\mat U_k, \mat X_L)\leq \epsilon$.
\end{lem}
\begin{proof}
	First, we write $\mat{X}_L$ as \begin{align*}
	\mat{X}_L = \mat{U}\mat{U}^\top\mat{X}_L = \mat{U}\begin{pmatrix}
	\mat{U}_q^\top\mat{X}_L\\
	\mat{U}_{d-q}^\top\mat{X}_L
	\end{pmatrix},
	\end{align*}
	where $\mat U$ is the orthogonal space of $\mat A$.
	Next, consider a $p\times q$ matrix  $\widehat{\mat X}$ that is orthogonal to $\left(\mat{U}_q^\top\mat{X}_L\right)$; that is, $\left(\mat{U}_q^\top\mat{X}_L\right)\widehat{\mat{X}} = \mat{0}$.
	Following the techniques introduced in \citep{mgu-subspace-iteration,halko2011finding}, we consider the following matrix:
	\begin{align*}
	\mat{X} =
	\left(\mat{U}_q^\top\mat{X}_L\right)^\dagger & \widehat{\mat{X}}.
	\end{align*}
	By definition, we then have that
	\begin{align*}
	\mat{X}_L\mat{X} = \mat{U} \begin{pmatrix}
	\mat{I} & \mat{0} &\mat{0} \\
	\mat{0} & \mat{I} & \mat{0} \\
	\mat{H}_1 & \mat{H}_2 & \mat{H}_3
	\end{pmatrix},
	\end{align*}
	where
	\begin{align*}
	\mat{H}_1 &= \left(\mat{U}_{d-q}^\top\mat{X}_L\right)\left(\mat{U}_q^\top\mat{X}_L\right)^\dagger\begin{pmatrix}
	\mat{I}_{k \times k}\\
	\mat{0}
	\end{pmatrix},\\
	\mat{H}_2 &= 
	\left(\mat{U}_{d-q}^\top\mat{X}_L\right)\left(\mat{U}_q^\top\mat{X}_L\right)^\dagger\begin{pmatrix}
	\mat{0}\\
	\mat{I}_{\left(q-k\right) \times \left(q-k\right)}
	\end{pmatrix},\\
	\mat{H}_3 &= \left(\mat{U}_{d-q}^\top\mat{X}_L\right)\widehat{\mat{X}}.
	\end{align*} 
	Note that $\norm{\mat{H}_1}_2 = h_L$ by definition.
	Under the condition of the lemma $h_L\leq\epsilon/4$, we have that $\|\mat H_1\|_2\leq\epsilon/4$.
	We next consider an alternative QR decomposition of $\mat X_L\mat X$:
	\begin{equation*}
	\mat{X}_L\mat{X}  = \widehat{\mat{Q}}\widehat{\mat{R}} 
	= \begin{pmatrix}
	\widehat{\mat{Q}}_1 & \widehat{\mat{Q}}_2 & \widehat{\mat{Q}}_3 
	\end{pmatrix}
	\begin{pmatrix}
	\widehat{\mat{R}}_{11} & \widehat{\mat{R}}_{12} & \widehat{\mat{R}}_{13} \\
	& \widehat{\mat{R}}_{22} & \widehat{\mat{R}}_{23} \\
	&  & \widehat{\mat{R}}_{33}
	\end{pmatrix}.
	\end{equation*}
	Because the projection matrix $\widehat{\mat Q}$ is unique, we have $\widehat{\mat{Q}}\widehat{\mat{Q}}^\top = \mat{X}_L\mat{X}_L^\top$.
	Also note that the above QR decomposition embeds another smaller one:
	\begin{align*}
	\mat{U}\begin{pmatrix}
	\mat{I}\\
	\mat{0}\\
	\mat{H}_1
	\end{pmatrix} = \widehat{\mat{Q}}_1 \widehat{\mat{R}}_{11}.
	\end{align*} 
	The projection operator orthogonal to $\widehat{\mat{Q}}_1$ can be expressed as
	\begin{align*}
	\mat{I} -\widehat{\mat{Q}}_1\widehat{\mat{Q}}_1^\top 
	& = \mat{U}\mat{U}^\top - \widehat{\mat{Q}}_1\widehat{\mat{Q}}_1^\top \\
	& = \mat{U}\begin{pmatrix}
	\mat{I} \\
	\mat{0}\\
	\mat{H}_1
	\end{pmatrix}\widehat{\mat{R}}_{11}^{-1}\widehat{\mat{R}}_{11}^{-\top}\begin{pmatrix}
	\mat{I} & \mat{0} & \mat{H}_1^\top
	\end{pmatrix}\mat{U}^\top\\
	& = \mat{U}\begin{pmatrix}
	\mat{I} - \left(\mat{I}+\mat{H}_1^\top\mat{H}_1\right)^{-1} & \mat{0} & -\left(\mat{I}+\mat{H}_1^\top\mat{H}_1\right)\mat{H}_1^\top\\
	\mat{0} & \mat{I} & \mat{0} \\
	-\mat{H}_1\left(\mat{I}+\mat{H}_1^\top\mat{H}_1\right)^{-1} & \mat{0} & \mat{I} - \mat{H}_1\left(\mat{I}+\mat{H}_1^\top\mat{H}_1\right)^{-1}\mat{H}_1^\top
	\end{pmatrix}\mat{U}^\top,
	\end{align*}
	where in the last equation we use the fact that $\widehat{\mat{R}}_{11}\widehat{\mat{R}}_{11} = \left(\mat{I} + \mat{H}_1^\top\mat{H}_1\right)^{-1}$. 
	The principal angle $\theta_k(\mat U_k, \mat X_L)$ can then be bounded as
	\begin{align*}
	\sin\theta_k\left(\mat{U}_k,\mat{X}_L\right) & = \norm{\left(\mat{I} - \mat{X}_L\mat{X}_L^\top\right)\mat{U}_k}_2\\
	& = \norm{\left(\mat{I} - \widehat{\mat{Q}}\widehat{\mat{Q}}^\top\right)\mat{U}_k}_2 \\
	& \le \norm{\left(\mat{I} - \widehat{\mat{Q}}_1\widehat{\mat{Q}}_1^\top\right)\mat{U}_k}_2\\
	& = \norm{
		\mat{U}\begin{pmatrix}
		\mat{I} - \left(\mat{I}+\mat{H}_1^\top\mat{H}_1\right)^{-1} & \mat{0} & -\left(\mat{I}+\mat{H}_1^\top\mat{H}_1\right)\mat{H}_1^\top\\
		\mat{0} & \mat{I} & \mat{0} \\
		-\mat{H}_1\left(\mat{I}+\mat{H}_1^\top\mat{H}_1\right)^{-1} & \mat{0} & \mat{I} - \mat{H}_1\left(\mat{I}+\mat{H}_1^\top\mat{H}_1\right)^{-1}\mat{H}_1^\top
		\end{pmatrix}\mat{U}^\top\mat{U}_k
	}_2\\
	& = \norm{
		\mat{U}\begin{pmatrix}
		\mat{I} - \left(\mat{I}+\mat{H}_1^\top\mat{H}_1\right)^{-1} & \mat{0} & -\left(\mat{I}+\mat{H}_1^\top\mat{H}_1\right)\mat{H}_1^\top\\
		\mat{0} & \mat{I} & \mat{0} \\
		-\mat{H}_1\left(\mat{I}+\mat{H}_1^\top\mat{H}_1\right)^{-1} & \mat{0} & \mat{I} - \mat{H}_1\left(\mat{I}+\mat{H}_1^\top\mat{H}_1\right)^{-1}\mat{H}_1^\top
		\end{pmatrix}\begin{pmatrix}
		\mat{I}_{k \times k}\\
		\mat{0}\\
		\mat{0}
		\end{pmatrix}
	}_2\\
	& \le \norm{\mat{I} - \left(\mat{I}+\mat{H}_1^\top\mat{H}_1\right)^{-1}}_2 + \norm{\mat{H}_1\left(\mat{I}+\mat{H}_1^\top\mat{H}_1\right)^{-1}}_2,
	\end{align*} where the first inequality is due to the space projected by $\widehat{\mat{Q}}_1\widehat{\mat{Q}}_1^\top$ is a subspace of that by $\widehat{\mat{Q}}\widehat{\mat{Q}}^\top$.
	By Woodbury's identity, we have that
	\begin{align*}
	\norm{\mat{I} - \left(\mat{I}+\mat{H}_1^\top\mat{H}_1\right)^{-1}}_2 = \norm{\mat{H}_1^\top\left(\mat{I}+\mat{H}_1\mat{H}_1^\top\right)\mat{H}_1}_2 \le \frac{(\epsilon/4)^2}{1-(\epsilon/4)^2} \le \epsilon/2.
	\end{align*}For the other term, we have 
	\begin{align*}
	\norm{\mat{H}_1\left(\mat{I}+\mat{H}_1^\top\mat{H}_1\right)^{-1}}_2 \le \frac{\epsilon/4}{1-(\epsilon/4)^2} \le \epsilon/2.
	\end{align*} Combing these two inequalities, we get 
	\begin{align*}
	\sin\theta_k\left(\mat{U}_k, \mat{X}_L \right)  \le \epsilon.
	\end{align*}
	The proof is then completed by noting that $\sin\theta_k\left(\mat{U}_k, \mat{X}_L \right)  \le \epsilon/2$ yields
	$$\tan\theta_k\left(\mat{U}_k, \mat{X}_L \right)  = \frac{\sin\theta_k\left(\mat{U}_k, \mat{X}_L \right) }{\sqrt{1-\sin^2\left(\mat{U}_k, \mat{X}_L\right) }} \le \epsilon.$$
\end{proof}

\begin{lem}\label{lem:angle_preserving}
	Fix $0< \gamma < 1$. 
	If at each iteration $\ell$ the noise matrix $\mat G_\ell$ satisfies
	\begin{align*}
	\norm{\mat{G}_\ell}_2 = O\left(\gamma \sigma_q\right) \quad \text{and} \quad \norm{\mat{U}_q^\top\mat{G}_\ell}_2 = O\left(\frac{\sqrt{p}-\sqrt{q-1}}{\tau\sqrt{d}}\cdot \gamma \sigma_q\right),
	\end{align*} 
	then for all $\ell = O\left(1/\gamma\right)$, 
	the following holds with probability all but $\tau^{-\Omega\left(p+1-q\right)} + e^{-\Omega\left(d\right)}$ probability:
	\begin{equation*}
	\tan\theta_q\left(\mat{U}_q,\mat{X}_\ell\right) = O\left(\frac{\tau\sqrt{d}}{\sqrt{p}-\sqrt{q-1}}\right), \;\;\;\;\;\;
	\cos\theta_q\left(\mat{U}_q,\mat{X}_\ell\right) = \Omega\left(\frac{\sqrt{p}-\sqrt{q-1}}{\tau\sqrt{d}}\right).
	\end{equation*}
\end{lem}
\begin{proof}
	By Lemma~\ref{thm:h_initialization}, the tangent of the $q$th principal angle between $\mat U_q$ and $\mat X_0$ can be bounded as
	\begin{equation}
	\tan\theta_q(\mat U_q,\mat X_0) \leq \frac{\tau\sqrt{d}}{\sqrt{p}-\sqrt{q-1}}
	\label{eq_tan_init}
	\end{equation}
	with high probability.
	We also consider the following inequality that upper bounds $\tan\theta_q(\mat U_q,\mat X_\ell)$ in terms of $\tan\theta_q(\mat U_q,\mat X_0)$:
	\begin{equation}
	\tan\theta_q\left(\mat{U}_q,\mat{X}_\ell\right) + \frac{c_1}{c_1+c_3}\le \left(\frac{1+c_1\gamma}{1-c_3\gamma}\right)^\ell\left(\tan\theta_q\left(\mat{U}_q,\mat{X}_0\right) +  \frac{c_1}{c_1+c_3}\right).
	\label{eq_tan_shrinkage}
	\end{equation}
	Here $c_1,c_2,c_3>0$ are universal constants.
	Eq.~(\ref{eq_tan_init}) and Eq.~(\ref{eq_tan_shrinkage}) imply $\tan\theta_q(\mat U_q,\mat X_\ell) = O(\frac{\tau\sqrt{d}}{\sqrt{p}-\sqrt{q-1}})$ for all $\ell=O(1/\gamma)$ because
	$$
	\left(\frac{1+c_1\gamma}{1-c_3\gamma}\right)^\ell  =  \left(1+\frac{\left(c_1+c_3\right)\gamma}{1-c_3\gamma}\right)^{\frac{\left(c_1+c_3\right)\gamma}{1-c_3\gamma}\cdot\left(\frac{1-c_3\gamma}{\left(c_1+c_3\right)\gamma}\right)\cdot \ell} 
	\le \exp\left(\frac{1-c_3\gamma}{\left(c_1+c_3\right)\gamma}\cdot \ell\right) 
	= O(1),
	$$
	if $\ell=O(1/\gamma)$. $\cos\theta_q(\mat U_q,\mat X_\ell)$ can subsequently be lower bounded as
	$$
	\cos\left(\mat{U}_q, \mat{X}_\ell\right) \ge \frac{1}{1+\tan\left(\mat{U}_q, \mat{X}_\ell\right)} = \Omega\left( \frac{\sqrt{p}-\sqrt{q-1}}{\tau\sqrt{d}}\right).
	$$
	
	The rest of the proof is dedicated to prove Eq.~(\ref{eq_tan_shrinkage}) via mathematical induction.
	When $\ell = 0$, the statement is trivially true.
	Suppose for Eq.~(\ref{eq_tan_shrinkage}) is true for all $\ell = 1,\cdots,s$. 
	We want to prove that Eq.~(\ref{eq_tan_shrinkage}) is also true for $\ell=s+1$.
	By definition, 
	\begin{equation*}
	\tan\theta_q\left(\mat{U}_q, \mat{X}_\ell\right) 
	= \min_{\mat\Pi\in\mathcal P_p}\max_{\norm{\mat{w}}=1,\mat{\Pi} \mat{w}=\mat{w}} \frac{\norm{\mat{U}_{d-q}^\top\mat{X}_\ell \mat{w}}}{\norm{\mat{U}_q^\top\mat{X}_\ell \mat{w}}}
	= \max_{\norm{\mat{w}}=1,\mat{\Pi}^\star \mat{w}=\mat{w}} \frac{\norm{\mat{U}_{d-q}^\top\mat{X}_\ell \mat{w}}}{\norm{\mat{U}_q^\top\mat{X}_\ell \mat{w}}}.
	\end{equation*} 
	Here $\mathcal P_p$ denotes the set of all projection matrices on $\mathbb R^p$ and $\mat\Pi^*$ is the projection matrix that achieves the minimum value in the second term.
	We then have \begin{align}
	\tan\theta_q\left(\mat{U}_q,\mat{X}_{\ell+1}\right) &= \tan\theta_q\left(\mat{U}_q, \mat{A}\mat{X}_\ell+\mat{G}_\ell\right) \nonumber\\
	& = \min_{\mat{\Pi} \in \mathcal{P}_p}\max_{\norm{\mat{w}}_2=1, \mat{\Pi}\mat{w}=\mat{w}} \frac{\norm{\mat{U}_{d-q}^\top\left(\mat{A}\mat{X}_\ell + \mat{G}_\ell\right)\mat{w}}}{\norm{\mat{U}_q^\top\left(\mat{A}\mat{X}_\ell+\mat{G}_\ell\right)\mat{w}}} \nonumber\\
	& \le \max_{\norm{\mat{w}}_2=1,\mat{\Pi}^\star\mat{w}=\mat{w}} \frac{
		\norm{
			\mat{\Sigma}_{d-q}\mat{U}_{d-q}^\top\mat{X}_\ell\mat{w}
		}_2 + \norm{\mat{U}_{d-q}\mat{G}_\ell \mat{w}}_2
	}{
	\norm{
		\mat{\Sigma}_q\mat{U}_q^\top\mat{X}_\ell\mat{w}
	}_2  - \norm{\mat{U}_q^\top\mat{G}_\ell\mat{w}}_2
} \nonumber \\
& \le \max_{\norm{\mat{w}}_2=1,\mat{\Pi}^\star\mat{w}=\mat{w}} \frac{
	\sigma_{q+1}\norm{\mat{U}_{d-q}^\top\mat{X}_\ell\mat{w}}_2/\norm{\mat{U}_q^\top\mat{X}_\ell w}_2 + \norm{\mat{G}_\ell}_2/\norm{\mat{U}_q^\top\mat{X}_\ell \mat{w}}_2
}{\sigma_q - \norm{
	\mat{U}_q^\top\mat{G}_\ell \mat{w}
}_2/ \norm{\mat{U}_q^\top \mat{X}_\ell \mat{w}}_2
} \label{eqn:principal_angel}
\end{align}
By definition of the principal angles, we have \begin{align*}
\max_{\norm{\mat{w}}_2=1,\mat{\Pi}^\star\mat{w}=\mat{w}} \norm{\mat{U}_{d-q}^\top\mat{X}_\ell\mat{w}}_2/\norm{\mat{U}_q^\top\mat{X}_\ell w}_2 = \tan\left(\mat{U}_q, \mat{X}_\ell\right), \\ 
\max_{\norm{\mat{w}}_2=1,\mat{\Pi}^\star\mat{w}=\mat{w}} 
\frac{1}{\norm{\mat{U}_q^\top \mat{X}_\ell \mat{w}}_2} = \frac{1}{\cos\left(\mat{U}_q, \mat{X}_\ell\right)} \le 1+ \tan\left(\mat{U}_q, \mat{X}_\ell\right).
\end{align*} 
Also, conditions on the noise matrices $\mat G_\ell$ read 
\begin{align*}
\norm{\mat{G}_\ell}_2 &= \le c_1\gamma\sigma_q, \;\;\;\;\;\;
\norm{\mat{U}_q^\top\mat{G}_\ell}_2 \le c_3\gamma \sigma_q \cos\left(\mat{U}_q, \mat{X}_\ell\right).
\end{align*} 
Plugging these inequalities into Eq.~(\ref{eqn:principal_angel}), we obtain
\begin{align*}
\tan\left(\mat{U}_q, \mat{X}_{\ell+1}\right) & \le \frac{\sigma_{q+1} \tan\left(\mat{U}_q,\mat{X}_\ell\right) + c_1\gamma\left(1+\tan\left(\mat{U}_q, \mat{X}_\ell\right)\right)}{\sigma_q - c_3\gamma\sigma_q} \\
& \le \left(\frac{1+c_1\gamma}{1-c_3\gamma}\right)\tan\left(\mat{U}_q, \mat{X}_\ell\right) + \frac{c_1\gamma}{1-c_3\gamma} \\
&\leq \left(\frac{1+c_1\gamma}{1-c_3\gamma}\right)^\ell\left(\tan\theta_q\left(\mat{U}_q,\mat{X}_0\right) +  \frac{c_1}{c_1+c_3}\right),
\end{align*} 
where the last inequality is due to induction hypothesis placed on Eq.~(\ref{eq_tan_shrinkage}).
\end{proof}

\begin{cor}\label{thm:angle_preserving}
	Fix $\epsilon = O\left(\frac{\sigma_q}{\sigma_k}\cdot\min\left\{\frac{1}{\log\left(\frac{\sigma_k}{\sigma_q}\right)},\frac{1}{\log\left(\tau d\right)}\right\}\right)$.
	Suppose at each iteration the noise matrix $\mat G_\ell$ satisfies \begin{align*}
	\norm{\mat{G}_\ell}_2 = O\left(\epsilon\left(\sigma_k - \sigma_{q+1}\right)\right) \quad \text{and} \quad \norm{\mat{U}_q^\top\mat{G}_\ell} = O\left(\frac{\sqrt{p}-\sqrt{q-1}}{\tau\sqrt{d}}\cdot \min\{\epsilon\left(\sigma_k - \sigma_{q+1}\right), \sigma_q\}\right),
	\end{align*} 
	then for all $\ell = O\left(\frac{\sigma_k}{\sigma_k - \sigma_{q+1}}\log\left(\frac{\tau d}{\epsilon}\right)\right)$
	the following holds with all but $\tau^{-\Omega\left(p+1-q\right)} + e^{-\Omega\left(d\right)}$ probability:
	\begin{equation*}
	\tan\theta_q\left(\mat{U}_q,\mat{X}_\ell\right) = O\left(\frac{\tau\sqrt{d}}{\sqrt{p}-\sqrt{q-1}}\right),\;\;\;\;\;\;
	\cos\theta_q\left(\mat{U}_q,\mat{X}_\ell\right) = \Omega\left(\frac{\sqrt{p}-\sqrt{q-1}}{\tau\sqrt{d}}\right).
	\end{equation*}
\end{cor}
\begin{proof}
	Apply Lemma~\ref{lem:angle_preserving} with $\gamma=\min\{\frac{\epsilon(\sigma_k-\sigma_{q+1})}{\sigma_q}, 1\}$.
\end{proof}

\section{Proof of Theorem \ref{thm:gap_independent}}\label{appsec:gapfree}

\begin{proof}
	Define $m = \argmax_{i}\{\sigma_i - \sigma_{k+1} \ge \epsilon\sigma_{k+1}\}$. If $m=0$, then we are done since $\norm{\mat{A} - \mat{X}_L\mat{X}_L^\top \mat{A}}_2\le \norm{\mat{A}}_2 \le \sigma_1 \le \left(1+\epsilon\right)\sigma_{k+1} = \left(1+\epsilon\right)\norm{\mat{A} -\mat{A}_k}_2$. Otherwise, consider the case that our target rank is $m$, and the leading rank-$k$ subspace. By our definition on $m$ and noise conditions,
	we have \begin{align*}
	\norm{\mat{G}}_2 & = O\left(\epsilon^2\sigma_{k+1}\right) = O\left(\epsilon\left(\sigma_m - \sigma_{k+1}\right)\right); \\
	\norm{\mat{U}_k^\top \mat{G}}_2 & = O \left( \frac{\epsilon^2\left(\sqrt{p}-\sqrt{k-1}\right)\sigma_{k+1}}{\tau\sqrt{d}} \right) = O \left( \frac{\epsilon\left(\sqrt{p}-\sqrt{k-1}\right)\left(\sigma_m -\sigma_{k+1}\right)}{\tau\sqrt{d}} \right).
	\end{align*} 
	Next, by Lemma~\ref{thm:angel_preserving_gap_independent}, for all $\ell = O\left(\frac{1}{\epsilon^2}\right)$
	the cosine principal angle $\cos\theta_q(\mat U_q,\mat X_\ell)$ can be lower bounded as
	\begin{align*}
	\cos\left(\mat{U}_k,\mat{X}_\ell\right) = \Omega\left(\frac{\sqrt{p}-\sqrt{k-1}}{\tau \sqrt{d}}\right).
	\end{align*}
	Note also that $\frac{\sigma_m}{\sigma_m - \sigma_{k+1}}\log\left(\frac{\tau d}{\epsilon}\right) \lesssim \frac{1}{\epsilon}\log\left(\frac{\tau d}{\epsilon }\right) \lesssim L$.
	Using the same argument as in Appendix~\ref{appsec:proof}, we have $\tan\theta_m\left(\mat{U}_m, \mat{X}_L\right) \le \epsilon$. 
	Applying Theorem 9.1 of~\citep{halko2011finding}, we obtain \begin{align*}
	\norm{\mat{A}- \mat{X}_L\mat{X}_L^\top \mat{A}}_2^2 & \le \left(1+ \tan\theta_m\left(\mat{U}_m, \mat{X}_l\right)^2 \right) \norm{\mat{A}-\mat{A}_m}_2^2 \\
	& \le \left(1+\epsilon^2\right)\sigma_{m+1}^2 \\
	& \le \left(1+\epsilon^2\right)\left(1+\epsilon\right)^2\sigma_{k+1}^2.
	\end{align*} 
	Rescaling the accuracy parameter $\epsilon$ we prove the desired result.
\end{proof}

\begin{lem}\label{thm:angel_preserving_gap_independent}
	Fix $\epsilon = O\left(1/{\log\left(\tau d\right)}\right)$.
	If at each iteration the noise matrix $\mat G_\ell$ satisfies 
	\begin{align*}
	\norm{\mat{G}_\ell}_2 = O\left(\epsilon^2\sigma_k\right) \quad \text{and} \quad \norm{\mat{U}_q^\top\mat{G}_\ell} = O\left(\frac{\sqrt{p}-\sqrt{q-1}}{\tau\sqrt{d}}\cdot \epsilon^2\sigma_k\right),
	\end{align*} 
	then for all $\ell = O\left(1/{\epsilon^2}\right)$ the following holds with all but $\tau^{-\Omega\left(p+1-q\right)} + e^{-\Omega\left(d\right)}$ probability:
	\begin{equation*}
	\tan\theta_q\left(\mat{U}_q,\mat{X}_\ell\right) = O\left(\frac{\tau\sqrt{d}}{\sqrt{p}-\sqrt{q-1}}\right), \;\;\;\;\;\;
	\cos\theta_q\left(\mat{U}_q,\mat{X}_\ell\right) = \Omega\left(\frac{\sqrt{p}-\sqrt{q-1}}{\tau\sqrt{d}}\right).
	\end{equation*}
\end{lem}
\begin{proof}
	Apply Lemma~\ref{lem:angle_preserving} with $p=k$ and $\gamma=\epsilon^2$.
\end{proof}

\section{Proof of results for distributed private PCA}\label{appsec:dppca}

\begin{thm}[Distributed private PCA, Theorem \ref{thm:dppca}]
	Let $s$ be the number of computing nodes and $\mat A^{(1)},\cdots,\mat A^{(s)}\in\mathbb R^{d\times d}$ be data matrices stored separately on the $s$ nodes.
	Fix target rank $k$, intermediate rank $q\geq k$ and iteration rank $p$ with $2q\leq p\leq d$.
	Suppose the number of iterations $L$ is set as $L=\Theta(\frac{\sigma_k}{\sigma_k-\sigma_{q+1}}\log(d))$.
	Let $\varepsilon,\delta\in(0,1)$ be privacy parameters.
	Then Algorithm \ref{alg_dppca} solves the $(\varepsilon,\delta,\epsilon,M)$-distributed PCA problem with
	$$
	\epsilon = O\left(\frac{\nu\sqrt{\mu(\mat A)s\log d\log L}}{\sigma_k-\sigma_{q+1}}\right)
	\quad\text{and}\quad
	M=O(spdL)=O\left(\frac{\sigma_k}{\sigma_k-\sigma_{q+1}}spd\log d\right).
	$$
	Here assuming conditions in Theorem~\ref{thm:new_main} are satisfied,  $\nu=\varepsilon^{-1}\sqrt{4pL\log(1/\delta)}$ and $\mu(\mat A)$ is the \emph{incoherence} \citep{private-pca-incoherent}
	of the aggregate data matrix $\mat A=\sum_{i=1}^s{\mat A^{(i)}}$;
	more specifically, $\mu(\mat A)=d\|\mat U\|_{\infty}$ where $\mat A=\mat U\mat\Lambda\mat U^{\top}$ is the eigen-decomposition of $\mat A$.
\end{thm}
\begin{proof}
	We prove privacy, utility and communication guarantees of Algorithm \ref{alg_dppca} separately.
	\paragraph{Privacy guarantee}
	By Claim 4.2 in \citep{noisy-power-method}, Algorithm \ref{alg_dppca} satisfies $(\varepsilon,\delta)$-differential privacy
	with respect to data matrix $\mat A^{(i)}$ on each computing node $i$.
	Because information of each data matrix $\mat A^{(i)}$ is only released by the corresponding computing node $i$ via the public communication channel,
	we immediately have that Algorithm \ref{alg_dppca} is $(\varepsilon,\delta)$-differentially private in terms of the definition in Eq.~(\ref{eq_privacy}).
	
	\paragraph{Utility guarantee}
	Let $\mat G_\ell = \mat G_\ell^{(1)} + \cdots + \mat G_\ell^{(s)}$.
	Because $\mat G_\ell^{(1)},\cdots,\mat G_\ell^{(s)}\overset{i.i.d.}{\sim}\nml(0, \|\mat X_{\ell-1}\|_\infty^2\nu^2)^{d\times p}$,
	we have that $\mat G_\ell\sim\nml(0,\|\mat X_{\ell-1}\|_\infty^2\tilde\nu^2)^{d\times p}$ for $\tilde\nu = \nu\sqrt{s}$.
	Properties of Gaussian matrices (e.g., Lemma A.2 in \citep{noisy-power-method}) show that with high probability
	$\mat G_\ell$ satisfies the noise conditions in Theorem~\ref{thm:new_main} with 
	$\epsilon=\frac{\nu\max_\ell\|\mat X_\ell\|_\infty\sqrt{ds\log L}}{\sigma_k-\sigma_{q+1}}$.
	In addition, Theorem 4.9 in \citep{noisy-power-method} shows that $\max_\ell\|\mat X_\ell\|_\infty^2 = O(\mu(\mat A)\log d/d)$ with high probability.
	The utility guarantee then holds by applying Theorem~\ref{thm:new_main} with bounds on $\epsilon$ and $\max_\ell\|\mat X_\ell\|_\infty^2$.
	
	\paragraph{Communication guarantee}
	For each iteration $\ell$, the central node broadcasts $\mat X_{\ell-1}$ to each computing node and receives $\mat A_\ell^{(i)}\mat X_{\ell-1}+\mat G_\ell^{(i)}$
	from computing node $i$, for each $i=1,\cdots, s$.
	Both matrices communicated on the public channel between the central node and each computing node is $d\times p$,
	which yields a per-iteration communication complexity of $O(spd)$.
	As a result, the total amount of communication is $O(spdL)$, where $L$ is the number of iterations carried out in Algorithm \ref{alg_dppca}.
	Because $L$ is set as $L=\Theta(\frac{\sigma_k}{\sigma_k-\sigma_{q+1}}\log d)$, we have that
	$
	M = O(spdL) = O\left(\frac{\sigma_k}{\sigma_k-\sigma_{q+1}}spd\log d\right).
	$
\end{proof}

\begin{cor}[Corollary \ref{cor_private_pca}]
	For the case of $s=1$ and $2p\leq q\leq d$, Algorithm \ref{alg_dppca} is $(\varepsilon,\delta)$-differentially private and $\mat X_L$ satisfies
	$$
	\|(\mat I-\mat X_L\mat X_L^\top)\mat U_k\|_2 \leq \epsilon = O\left(\frac{\nu\sqrt{\mu(\mat A)\log d\log L}}{\sigma_k-\sigma_{q+1}}\right)
	$$
	with probability at least 0.9. Here $\mat U_k$ is the top-$k$ eigen-space of input data matrix $\mat A\in\mathbb R^{d\times d}$.
\end{cor}
\begin{proof}
	Setting $s=1$ in Theorem \ref{thm:dppca} we immediately get this corollary.
\end{proof}

\begin{cor}[Corollary \ref{cor_distributed_pca}]
	Fix error tolerance parameter $\epsilon\in(0,1)$ and
	set $\nu=0$, $L=\Theta(\frac{\sigma_k}{\sigma_k-\sigma_{q+1}}\log(d/\epsilon))$ in Algorithm \ref{alg_dppca}. 
	We then have that with probability 1
	$$
	\|(\mat I-\mat X_L\mat X_L^\top)\mat U_k\|_2 \leq \epsilon.
	$$
	Here $\mat U_k$ is the top-$k$ eigen-space of the aggregated matrix $\mat A=\sum_{i=1}^{s}{\mat A^{(i)}}$.
\end{cor}
\begin{proof}
	Because $\nu=0$, we are not adding any amount of noise in Algorithm \ref{alg_dppca}; that is, $\mat G_\ell = \mat 0$.
	Applying Theorem \ref{thm:new_main} with $\mat G_\ell=\mat 0$ and $L=\Theta(\frac{\sigma_k}{\sigma_k-\sigma_{q+1}}\log (d/\epsilon))$ we have
	$\|(\mat I-\mat X_L\mat X_L^\top)\mat U_k\|_2\leq \epsilon$ with high probability.
\end{proof}

\end{document}